\documentclass[12pt]{l4dc2022}

\usepackage{todonotes}
\usepackage{thmtools,thm-restate}
\usepackage{algorithm}
\usepackage{algorithmic}
\usepackage{cleveref}
\newtheorem{assumption}{Assumption}

\newcommand{\ubar}[1]{\text{\b{$#1$}}}
\newcommand{\Ahat}{\hat{A}}
\newcommand\ahat{\hat{a}}
\newcommand\Bhat{\hat{B}}
\newcommand\bhat{\hat{b}}

\newcommand\Shat{\hat{S}}
\newcommand\Khat{\hat{K}}

\newcommand{\xhat}{\hat{x}}

\newcommand{\uhat}{\hat{u}}
\newcommand\Vhat{\hat{V}}
\newcommand\KCE{K^\mathsf{MM}}
\newcommand\PCE{P^\mathsf{MM}}
\newcommand\pce{p^\mathsf{MM}}
\newcommand\fCE{f^\mathsf{MM}}
\newcommand\VCE{V^\mathsf{MM}}
\newcommand\KILC{K^\mathsf{ILC}}
\newcommand\PILC{P^\mathsf{ILC}}
\newcommand\pilc{p^\mathsf{ILC}}
\newcommand\fILC{f^\mathsf{ILC}}
\newcommand{\Tr}{\mathrm{Tr}}
\newcommand\KOPT{K^\star}
\newcommand\POPT{P^\star}
\newcommand\popt{p^\star}
\newcommand\VOPT{V^\star}

\newcommand\AOPT{A^\star}
\newcommand\QOPT{Q^\star}
\newcommand\reals{\mathbb{R}}
\newcommand\epsA{\epsilon_A}
\newcommand\epsB{\epsilon_B}
\newcommand\epsP{\epsilon_P}
\newcommand\epsa{\epsilon_a}
\newcommand\epsb{\epsilon_b}

\newcommand\ILC{\textsf{ILC}}
\newcommand\MM{\textsf{MM}}

\title[On the Effectiveness of ILC]{On the Effectiveness of Iterative
  Learning Control}
\usepackage{times}

\author{%
 \Name{Anirudh Vemula} \Email{vemula@cmu.edu}\\
 \addr Robotics Institute, Carnegie Mellon University
 \AND
 \Name{Wen Sun} \Email{ws455@cornell.edu}\\
 \addr Department of Computer Science, Cornell University
 \AND
 \Name{Maxim Likhachev} \Email{maxim@cs.cmu.edu}\\
 \addr Robotics Institute, Carnegie Mellon University
 \AND
 \Name{J. Andrew Bagnell} \Email{dbagnell@ri.cmu.edu}\\
 \addr Aurora Innovation
}

\begin{document}

\maketitle

\begin{abstract}%
 Iterative learning control (\ILC) is a powerful technique for high performance
  tracking in the presence of modeling errors for optimal control applications.
  There is extensive prior work showing its empirical effectiveness in
  applications such as chemical reactors, industrial robots and quadcopters.
  However, there is little prior theoretical work that explains the effectiveness of
  \ILC{} even in the presence of large modeling errors, where optimal
  control methods using the misspecified model (\MM{})
  often perform poorly. Our work presents such a theoretical study of the performance
  of both \ILC{} and \MM{} on Linear Quadratic Regulator (LQR) problems with unknown
  transition dynamics. We show that the suboptimality gap, as measured with
  respect to the optimal LQR controller, for \ILC{} is lower than that for \MM{} by
  higher order
  terms that become significant in the regime of high modeling
  errors.
  A key part of our analysis is the perturbation bounds for the discrete Ricatti
  equation in the finite horizon setting, where the solution is not a fixed
  point and requires tracking the error using recursive bounds. We back our
  theoretical findings with empirical experiments on a toy linear dynamical
  system with an approximate model, a nonlinear inverted pendulum system with
  misspecified mass, and a nonlinear planar quadrotor system in the presence of
  wind. Experiments show that \ILC{} outperforms \MM{} significantly, in terms of the cost of
  computed trajectories, when modeling errors are high.
\end{abstract}


\section{Introduction}
\label{sec:intro}

Iterative learning control (\ILC{}) has seen widespread adoption in a range of control
applications where the dynamics of the system are subject to unknown
disturbances or in instances where model parameters are
misspecified~\cite{moore99}. While traditional feedback-based control methods
have been successful at tackling non-repetitive noise, \ILC{} has shown itself to
be effective at adjusting to repetitive disturbance through feedforward control
adjustment~\cite{arimoto84}. This was shown empirically in several robotic
applications such as manipulation~\cite{kuc91}, and quadcopter trajectory
tracking~\cite{schoellig12a, schoellig12b} among others.
Prior work~\cite{atkeson88} uses fixed point theory to analyze the
conditions for convergence of \ILC{} but does not present performance
bounds at convergence. Very recent
work~\cite{DBLP:conf/icml/AgarwalHMS21} presented a \ILC{} algorithm
that is robust to model mismatch and uncertainty. However, they
analyze the algorithm using planning regret, which measures regret
with respect to the best open loop plan in hindsight, and do not
study how the performance depends on modeling error.
Our work contributes to understanding the effectiveness of \ILC{}
by studying its worst case performance, as a function of modeling
error, in the linear quadratic
regulator (LQR) setting with unknown transition dynamics
and access to
an approximate model of the dynamics.

A simple approach to the LQR problem with an approximate
model of the dynamics is
to do optimal control using the misspecified model (\MM{}.) The
resulting controller is similar to the certainty equivalent controller
obtained by performing optimal control on estimated parameters of the
regulator and ignoring the uncertainty of the estimates in adaptive
control~\cite{astrom13}.
Despite the
simplicity of \MM{}, it is challenging to quantify its suboptimality, with
respect to the optimal LQR controller, as a result of the
modeling errors in the approximate model.


Our first contribution is proving worst case cost suboptimality bounds for
\MM{} in the finite horizon LQR setting in terms of the modeling error. This
requires us to depart from the fixed point analysis used in prior
work~\cite{mania19, konstantinov93}, as the
solution to the discrete Ricatti equation in the finite horizon is not a fixed
point. A key part of our analysis is establishing perturbation bounds by
carefully tracking the effect of modeling error through the horizon of the
control task. This allows us to quantify the worst case suboptimality gap of \MM{} in the
finite horizon LQR setting.

The second contribution is to utilize the same proof techniques as we used for
\MM{} to analyze the
suboptimality gap of \ILC{}. This allows us to explicitly compare the
worst case performance of
\ILC{} and \MM{} for LQR problems, and understand why \ILC{} works well in the regime of
large modeling errors when \MM{} often performs poorly.
Our analysis highlights that the suboptimality gap for \ILC{} is lower than that
for \MM{} by higher order terms that can become significant
when modeling errors are high. We also show that \ILC{} is capable of keeping the
system stable and cost from blowing up even in the presence of large modeling
errors, which \MM{} is incapable of. By interpreting the worst case
bounds, we identify several linear systems
with key characteristics that enable \ILC{} to be robust to large model
misspecifications, whereas \MM{} is unable to deal with model errors and
results in poor solutions.

The final contribution of this work is to present simple empirical experiments
involving optimal control tasks with linear and nonlinear dynamical systems that
back the theoretical findings from our analysis. The experiment results
reinforce our finding that in the regime of large modeling errors, \ILC{} performs
better than \MM{} and synthesizes control inputs that result in smaller suboptimality
gaps.

\section{Problem Setup}
\label{sec:problem-setup}

We consider the finite horizon linear quadratic regulator (LQR) setting with a
horizon $H$ and a fixed initial state $x_{0} \in \reals^{n}$. The dynamics of
the system are described by unknown matrices $A_{t} \in \reals^{n\times n}$ and
$B_{t} \in \reals^{n\times d}$ for $t=0, \cdots, H-1$ as follows:
$x_{t+1} = A_{t}x_{t} + B_{t}u_{t}$
where $u_{t} \in \reals^{d}$ is the control input at time step $t$. Any sequence
of control inputs $(u_{0}, \cdots, u_{H-1})$ results in a state trajectory
$(x_{0}, \cdots, x_{H})$. The cost function is defined using matrices
$Q \in \reals^{n\times n}$, $Q_{f} \in \reals^{n\times n}$ and $R \in \reals^{d \times d}$ as follows:
\begin{equation}
  \label{eq:7}
  V_{0}(x_{0}) = \sum_{t=0}^{H-1} x_{t}^{T}Qx_{t} + u_{t}^{T}Ru_{t} + x_{H}^{T}Q_{f}x_{H}
\end{equation}

From optimal control literature~\cite{anderson07}, we know that the above cost is minimized by
a linear time-varying state-feedback controller $\KOPT = (\KOPT_{0}, \cdots, \KOPT_{H-1})$
with control inputs $u_{t} = \KOPT_{t}x_{t}$ satisfying:
\begin{align*}
  \KOPT_{t} &= -(R + B_{t}^{T}\POPT_{t+1}B_{t})^{-1}B_{t}^{T}\POPT_{t+1}A_{t} \\
  \POPT_{t} &= Q + A_{t}^{T}\POPT_{t+1}(I + B_{t}R^{-1}B_{t}^{T}\POPT_{t+1})^{-1}A_{t}
\end{align*}
where we initialize $\POPT_{H} = Q_{f}$ and the matrices $\POPT_{t}$ define the optimal
cost-to-go incurred using the optimal controller $\KOPT$ from time step $t$ as
$\VOPT_{t}(x_{t}) = x_{t}^{T}\POPT_{t}x_{t}$. For any controller $K$, we will
use the notation $M_{t}(K)$ to denote the matrix $A_{t} + B_{t}K_{t}$, and the
notation $L_{t}(K)$ to denote the product $\prod_{i=0}^{t}M_{i}(K)$. This is
useful for conciseness as we can observe that the state trajectory obtained
using $K$ can be expressed as $x_{t} = M_{t-1}(K)x_{t-1} = L_{t-1}(K)x_{0}$.

We are given access to an approximate model of the dynamics of the system
specified by matrices $\Ahat_{t} \in \reals^{n\times n}$ and
$\Bhat_{t} \in \reals^{n \times d}$ for $t = 0, \cdots, H-1$ such that there
exists some $\epsA, \epsB \geq 0$ (also referred to as the modeling error) satisfying
$||A_{t} - \Ahat_{t}|| \leq \epsA$ and $||B_{t} - \Bhat_{t}|| \leq \epsB$.
For the purposes of this paper, we use the notation $||\cdot||$ to refer to the
matrix norm induced by the L2 vector norm.
In this paper, we consider two control strategies: optimal control
using the misspecified model
(\MM{}) and iterative learning control (\ILC{}.)

\subsection{Optimal Control using Misspecified Model}
\label{sec:cert-equiv-contr}

Optimal control using misspecified model uses the
approximate model to synthesize a time-varying linear controller
$\KCE = (\KCE_{0}, \cdots, \KCE_{H-1})$ satisfying:
\begin{align*}
  \KCE_{t} &= -(R + \Bhat_{t}^{T}\PCE_{t+1}\Bhat_{t})^{-1}\Bhat_{t}^{T}\PCE_{t+1}\Ahat_{t} \\
  \PCE_{t} &= Q + \Ahat_{t}^{T}\PCE_{t+1}(I + \Bhat_{t} R^{-1}\Bhat_{t}^{T}\PCE_{t+1})^{-1}\Ahat_{t}
\end{align*}
where we initialize $\PCE_{H} = Q_{f}$ and the control inputs are defined as
$u_{t}^{\mathsf{CE}} = \KCE_{t}x_{t}$. One can observe that the controller
$\KCE$ results in suboptimal cost when executed in the system as it is
optimizing the cost under approximate
dynamics rather than the true dynamics of the system. Thus, the suboptimality
gap $\VCE_{0}(x_{0}) - \VOPT_{0}(x_{0})$ depends on the approximate dynamics
$\Ahat_{t}, \Bhat_{t}$, and how well they approximate the true dynamics.

\subsection{Iterative Learning Control}
\label{sec:iter-learn-contr-1}

Iterative learning control~\cite{arimoto84, moore99} is a framework
that is used to efficiently calculate the feedforward input signal
adjustment by using information from previous trials to improve the
performance in a small number of iterations. An example of an \ILC{}
algorithm is shown in Algorithm~\ref{alg:ilc}.
\ILC{} assumes a rollout access to
the system, i.e. we are
allowed to conduct full rollouts of horizon $H$ in the system to evaluate the
cost and obtain the trajectory under true dynamics
(Line~\ref{line:rollout}). Note that this access is
only restricted to rollouts, and the true dynamics $A_{t}, B_{t}$ are unknown.
\ILC{} can be understood as an iterative shooting method where we
synthesize control inputs by 
always evaluating in the true system while computing updates to the controls
using the approximate model~\cite{abbeel06, DBLP:conf/icml/AgarwalHMS21}. In
Algorithm~\ref{alg:ilc}, this is achieved by linearizing the
dynamics and quadraticizing the cost around the observed
trajectory (Line~\ref{line:lqr}) resulting in an LQR problem with the objective:
\begin{equation}
  \label{eq:5}
  J(\Delta x, \Delta u) = \sum_{t=0}^{H-1} (2x_t + \Delta x_t)^TQ\Delta x_t +
    (2u_t + \Delta u_t)^TR\Delta u_t + (2x_H + \Delta x_H)^TQ_f\Delta x_H 
  \end{equation}
  where $x_{0:H}$ is the observed trajectory on the true system when
  executing controls $u_{0:H-1}$, and for any $t=0, \cdots, H-1$ we
  have $\Ahat_t\Delta x_t + \Bhat_t\Delta u_t = \Delta x_{t+1}$.
  \begin{algorithm}[t]
    \small
  \caption{\ILC{} Algorithm for Linear Dynamical System with
    Approximate Model}
  \label{alg:ilc}
  \begin{algorithmic}[1]
    \STATE {\bfseries Input:} Approximate model $\Ahat_t,
    \Bhat_t$, Initial state $x_0$, Step size $\alpha$, cost matrix
    $Q, R, Q_f$
    \STATE Initialize a control sequence $u_{0:H-1}$ using approximate
    model
    \WHILE {not converged}
    \STATE Rollout $u_{0:H-1}$ on the true system to get trajectory
    $x_{0:H}$\label{line:rollout}
    \STATE Compute LQR solution $\arg\min_{\Delta x, \Delta u} J(\Delta x, \Delta u)$ subject to $\Ahat_t\Delta x_t + \Bhat_t\Delta u_t = \Delta x_{t+1}$ \label{line:lqr}
    \STATE Update $u_{0:H-1} = u_{0:H-1} + \alpha \Delta u_{0:H-1}$
    \ENDWHILE
  \end{algorithmic}
\end{algorithm}

At convergence in Algorithm~\ref{alg:ilc}, we have $\Delta u = 0$,
i.e. the LQR problem in line~\ref{line:lqr} returns the solution where
$\Delta u = 0$. The solution to the LQR problem can be derived in
closed form using dynamic programming, and for any $t \in \{0, \cdots,
H-1\}$ is given by
\begin{align*}
  \Delta u_t = -(R + \Bhat_t^T\PILC_{t+1}\Bhat_t)^{-1}(Ru_t + \Bhat_t^T\PILC_{t+1}x_{t+1})
\end{align*}
where $\PILC_{t+1}$ captures the cost-to-go from time step $t+1$ with
$\PILC_H = Q_f$. To obtain $\Delta u_t = 0$ for any $t \in \{0, \cdots, H-1\}$,
it is necessary for the following condition to hold,
\begin{align*}
  &Ru_{t} + \Bhat^{T}_{t}\PILC_{t+1}x_{t+1} = 0 \\
  &Ru_{t} + \Bhat^{T}_{t}\PILC_{t+1}(A_{t}x_{t} + B_{t}u_{t}) = 0 \\
  &u_{t} = -(R + \Bhat_{t}^{T}\PILC_{t+1}B_{t})^{-1}\Bhat_{t}^{T}\PILC_{t+1}A_{t}x_{t}
\end{align*}
where we use the rollout trajectory to obtain
$x_{t+1} = A_{t}x_{t} + B_{t}u_{t}$. It is important to note that converging to
these control inputs require carefully chosing appropriate step sizes $\alpha$ at each
iteration in the \ILC{} Algorithm~\ref{alg:ilc}. Thus, we can see that
the control inputs $u_{0:H-1}$
\ILC{} converges to can be described using a time-varying state-feedback linear
controller $\KILC$ defined as:
\begin{align*}
  \KILC_{t} &= -(R + \Bhat_{t}^{T}\PILC_{t+1}B_{t})^{-1}\Bhat_{t}^{T}\PILC_{t+1}A_{t} \\
  \PILC_{t} &=  Q + \Ahat_{t}^{T}\PILC_{t+1}(I + B_{t}R^{-1}\Bhat_{t}^{T}\PILC_{t+1})^{-1}A_{t}
\end{align*}
where we initialize $\PILC_{H} = Q_{f}$ and the control inputs are defined as
$u^{\mathsf{ILC}}_{t} = \KILC_{t}x_{t}$. We can observe that the \ILC{} converges
to control inputs that are different from the ones computed by the optimal
controller $\KOPT$, and hence achieves suboptimal cost. In the next few sections, we
will analyze the suboptimality bounds for both \MM{} and \ILC{}, and show how \ILC{}
converges to control sequence that achieves lower costs and is more robust to
high modeling errors when compared to \MM{}.

\subsection{Assumptions}
\label{sec:assumptions}
In this section, we will present all the assumptions used in our analysis.
Our first assumption is on the cost matrices $Q, Q_f$ and $R$, also used in~\cite{mania19}:
\begin{assumption}
  We assume that $Q, Q_f$, and $R$ are positive-definite matrices. Note that simply
  scaling all of $Q, Q_f$, and $R$ does not change the optimal controller $\KOPT$, so we can
  assume that the smallest singular value of $R$, $\ubar{\sigma}(R) \geq 1$.
  \label{assumption:singularvalue}
\end{assumption}
The above assumption allows us to ignore terms relating to singular
values of $R$ in the analysis, keeping it concise. The
next assumption states that the true system is stable under the
optimal controller $\KOPT$. Similar notions of stability have been considered in~\cite{cohen18}:
\begin{assumption}
  We assume that the optimal controller $\KOPT$ satisfies
  $||A_{t} + B_{t}\KOPT_{t}|| \leq 1 - \delta$ for some
  $0 < \delta \leq 1$ and all $t=0, \cdots, H-1$.
  \label{assumption:stability}
\end{assumption}
Observe that the above assumption implies that $||M_{t}(\KOPT)|| \leq 1 - \delta$
and $||L_{t}(\KOPT)|| \leq (1 - \delta)^{t+1} \leq e^{{-\delta(t+1)}}$.
Finally, we make a crucial assumption about the model that
is required for our \ILC{} analysis:
\begin{assumption}
  We assume that the matrix $B_{t}R^{-1}\Bhat_{t}^{T}$ has eigenvalues that have
  non-negative real parts for all $t=0, \cdots, H-1$. A sufficient condition for
  this to hold is that the
  modeling error satisfy
  $\epsB \leq \frac{\ubar{\sigma}(B_{t}^{T}RB_{t})}{||B_{t}^{T}R||}$
  for all $t=0, \cdots, H-1$.
  \label{assumption:psd}
\end{assumption}
The above assumption ensures that $x^{T}B_{t}R^{-1}\Bhat_{t}^{T}x \geq 0$ for any vector
$x \in \reals^{n}$ and all time steps $t$. Intuitively, if this is not true then \ILC{} is not
guaranteed to converge to a local minima. A more detailed
explanation is given in Appendix~\ref{sec:assumpt-refass}.

\section{Main Results}
\label{sec:main-results}

In this section, we will present the main results concerning the worst
case performance bounds of \MM{} and \ILC{} in the LQR setting with an
approximate model as described in Section~\ref{sec:problem-setup}. Our
first theorem (proof in Appendix~\ref{sec:general-results-1})
bounds the cost suboptimality of any time-varying linear 
controller $\Khat$ in terms of the norm differences $\|\KOPT_t -
\Khat_t\|$:
\begin{restatable}{theorem}{costTheorem}
  \label{theorem:cost}
  Suppose $d \leq n$. Denote
  $\Gamma = 1 + \max_t\{||A_t||, ||B_t||, ||\POPT_t||, ||\KOPT_t||\}$. Then under
  Assumption~\ref{assumption:stability} and if $||\KOPT_{t} -
  \Khat_{t}|| \le \frac{\delta}{2||B_{i}||}$ for all $t=0, \cdots, H-1$, we have
  \begin{equation}
    \label{eq:cost}
    \Vhat_0(x_0) - \VOPT_0(x_0) \leq d\Gamma^{3}\|x_{0}\|^{2} \sum_{t=0}^{H-1} e^{-\delta t}\|\KOPT_{t} - \Khat_{t}\|^{2}
  \end{equation}
\end{restatable}
This theorem is central to our analysis as it states that as long as we can keep
the norm differences $||\Khat_{t} - \KOPT_{t}||$ small, then the cost
suboptimality scales 
with the norm difference squared at each time step and goes exponentially down
with time step. 
We will now present results on how we can
bound these norm differences for both \MM{} and \ILC{}.

\paragraph{Results for Optimal Control with Misspecified Model}
\label{sec:optimal-control-with}

Our next
lemma (proof in Appendix~\ref{sec:cert-equiv-contr-2}) bounds the difference $\|\KCE_t - \KOPT_t\|$ in terms of
$\|\PCE_{t+1} - \POPT_{t+1}\|$ and modeling errors $\epsA, \epsB$:
\begin{restatable}{lemma}{ceLemma}
  \label{lemma:ce}
  If $||A_{t} - \Ahat_{t}|| \leq \epsA$ and
  $||B_{t} - \Bhat_{t}|| \leq \epsB$ for $t=0, \cdots, H-1$, and we have
  $||\POPT_{t+1} - \PCE_{t+1}|| \leq \fCE_{t+1}(\epsA, \epsB)$ for some function
  $\fCE_{t+1}$. Then we have under
  Assumption~\ref{assumption:singularvalue} for all $t=0, \cdots, H-1$,
  \begin{align}
    ||\KOPT_{t} - \KCE_{t}|| \leq 14\Gamma^3\epsilon_t
    \label{eq:ce-k-diff}
  \end{align}
  where
  $\Gamma = 1 + \max_{t}\{||A_{t}||, ||B_{t}||, ||\POPT_{t}||, ||\KOPT_{t}||\}$
  and $\epsilon_{t} = \max\{\epsA, \epsB, \fCE_{t+1}(\epsA, \epsB)\}$.
\end{restatable}
This result is very promising but there is a big piece still missing: how do we
bound $\fCE_{t+1}(\epsA, \epsB)$. To do this, we need to establish perturbation
bounds for the discrete ricatti equation in the finite
horizon setting. Prior work~\cite{konstantinov93, mania19} has only established such
bounds in the infinite horizon setting using fixed point analysis. Our treatment
is significantly different as the finite horizon solution is not a fixed point.
Our final perturbation bounds are presented in the theorem (proof in
Appendix~\ref{sec:cert-equiv-contr-2}) below:
\begin{restatable}{theorem}{theoremCE}
  \label{theorem:ce}
  If the cost-to-go matrices for the optimal controller and \MM{}
  controller are specified by $\{\POPT_t\}$ and $\{\PCE_t\}$ such that
  $\POPT_H = \PCE_H = Q_f$ then,
  \begin{align}
    \label{eq:14}
    ||\POPT_{t} - \PCE_{t}|| &\leq  \|A_{t}\|^{2}\|\POPT_{t+1}\|^{2}(2\|B_{t}\|\|R^{-1}\|\epsB + \|R^{-1}\|\epsB^{2}) \nonumber\\
    &+ 2\|A_{t}\|\|\POPT_{t+1}\|\epsA + \|\POPT_{t+1}\|\epsA^{2} \nonumber\\
                        &+ c_{\POPT_{t+1}}(\|A_{t}\|+ \epsA)^{2}||\POPT_{t+1} - \PCE_{{t+1}}||
  \end{align}
  for $t=0, \cdots, H-1$ where $c_{\POPT_{t+1}} \in \reals^+$  is a constant that is
  dependent only on $\POPT_{t+1}$ if $\epsA, \epsB$ are small enough such
  that $\|\POPT_{t+1} - \PCE_{t+1}\| \leq \|\POPT_{t+1}\|^{-1}$. Furthermore, the upper
  bound~\eqref{eq:14} is tight up to constants that only depend on the
  true dynamics $A_t, B_t$, cost matrix $R$, and $\POPT_{t+1}$.
\end{restatable}

The above theorem gives us an upper bound for $\fCE_{t}$ for $t=0,
\cdots, H-1$ in
Lemma~\ref{lemma:ce} with $\fCE_H = 0$. The resulting upper bound on $\|\KOPT_t -
\KCE_t\|$ from Lemma~\ref{lemma:ce} combined with
Theorem~\ref{theorem:cost} gives us the cost 
suboptimality bound for \MM{}. Notice that the
bound on $\fCE_t$ grows quickly as $t$ decreases making $\fCE_{t+1}$
in Lemma~\ref{lemma:ce} the dominant error term that affects the cost
suboptimality of \MM{}.

\paragraph{Results for Iterative Learning Control}
\label{sec:iter-learn-contr-3}

Our final set of results establish similar worst case cost
suboptimality bounds for \ILC{} by first establishing a bound (proof
in Appendix~\ref{sec:iter-learn-contr}) on the
difference $\|\KILC_t - \KOPT_t\|$ in terms of $\|\PILC_{t+1} -
\POPT_{t+1}\|$ and modeling error $\epsA, \epsB$:
\begin{restatable}{lemma}{ilcLemma}
  \label{lemma:ilc}
  If $||A_{t} - \Ahat_{t}|| \leq \epsA$ and
  $||B_{t} - \Bhat_{t}|| \leq \epsB$ for $t=0, \cdots, H-1$, and we have
  $||P_{t+1} - \PILC_{t+1}|| \leq \fILC_{t+1}(\epsA, \epsB)$ for some function
  $\fILC_{t+1}$. Then we have under
  Assumption~\ref{assumption:singularvalue} for all $t=0, \cdots, H-1$,
  \begin{align}
    ||\KOPT_{t} - \KILC_{t}|| \leq 6\Gamma^3\epsilon_t
    \label{eq:ilc-k-diff}
  \end{align}
  where
  $\Gamma = 1 + \max_{t}\{||A_{t}||, ||B_{t}||, ||\POPT_{t}||, ||\KOPT_{t}||\}$
  and $\epsilon_{t} = \max\{\epsA, \epsB, \fILC_{t+1}(\epsA,
  \epsB)\}$. 
\end{restatable}
Similar to \MM{}, we need to bound the crucial term
$\fILC_{t+1}(\epsA, \epsB)$ to bound the norm difference $\|\KOPT_t -
\KILC_t\|$ using Lemma~\ref{lemma:ilc}. We will present perturbation
bounds (proof in Appendix~\ref{sec:iter-learn-contr}) for the \ILC{}
recursion equation given in 
Section~\ref{sec:iter-learn-contr-1} in the finite horizon setting
below:
\begin{restatable}{theorem}{ilcTheorem}
  \label{theorem:ilc}
  If the cost-to-go matrices for the optimal controller and iterative learning
  control are specified by $\{\POPT_{t}\}$ and $\{\PILC_{t}\}$ such
  that $\POPT_{H} = \PILC_{H} = Q_f$ then we have under
  Assumption~\ref{assumption:psd},
  \begin{align}
    \label{eq:15}
    ||\POPT_{t} - \PILC_{t}|| &\leq  \|A_{t}\|^{2}\|\POPT_{t+1}\|^{2}\|B_{t}\|\|R^{-1}\|\epsB
    + \|A_{t}\|\|\POPT_{t+1}\|\epsA \nonumber\\
                        &+ c_{\POPT_{t+1}}||A_{t}||(\|A_{t}\|+ \epsA)||\POPT_{t+1} - \PILC_{{t+1}}||
  \end{align}
  for $t=0, \cdots, H-1$ where $c_{\POPT_{t+1}} \in \reals^+$ is a
  constant that is dependent only on 
  $\POPT_{t+1}$ if $\epsA, \epsB$ are small enough that $\|\POPT_{t+1}
  - \PILC_{t+1}\| \leq \|\POPT_{t+1}\|^{-1}$. Furthermore, the upper bound~\eqref{eq:15} is tight
  upto constants that depend only on the true dynamics $A_t, B_t$,
  cost matrix $R$, and $\POPT_{t+1}$. 
\end{restatable}
The above theorem gives us a bound on $\fILC_t$ for $t=0, \cdots, H-1$
in Lemma~\ref{lemma:ilc} with $\fILC_H = 0$. The resulting upper bound
on $\|\KOPT_t - \KILC_t\|$ from Lemma~\ref{lemma:ilc} combined with
Theorem~\ref{theorem:cost} gives us the cost suboptimality bound for
iterative learning control. Similar to \MM{}, the dominant error term in
Lemma~\ref{lemma:ilc} turns out to be $\fILC_{t+1}$ especially for
smaller $t$ as the upper bound~\eqref{eq:15} grows quickly as $t$ decreases.

\section{Interpreting the Worst Case Bounds}
\label{sec:interpr-worst-case}

The recursive bounds
presented in~\eqref{eq:14} and~\eqref{eq:15} make it difficult to
compute a concise bound in Theorem~\ref{theorem:cost}. In this
section, we will explicitly compare the cost suboptimality
bounds for \MM{} and \ILC{} under different scenarios, where the bound can
be simplified.

\paragraph{Small Modeling Errors}
\label{sec:small-model-errors}

In the regime of small modeling errors $\epsA << 1$ and $\epsB << 1$,
we can ignore quadratic terms $\epsA^2$ and $\epsB^2$ in upper
bound for \MM{}~\eqref{eq:14} which results in an upper bound that
matches that of \ILC{}~\eqref{eq:15} upto a constant. This suggests that
when the modeling errors are small, both \ILC{} and \MM{} have almost the
same worst case performance, with \ILC{} having better performance over
\MM{} by a constant factor. Intuitively, this makes sense as the approximate model is
a very good approximation of the true dynamics, and despite using only the
model, \MM{} can synthesize a near-optimal controller.

\paragraph{Highly Damped Systems}
\label{sec:highly-stable-system}

The second scenario we consider is that of a system that is highly damped
which implies $\|A_t\| << 1$ for all $t=0, \cdots,
H-1$. In this regime, the upper bound for \ILC{}~\eqref{eq:15} goes down
to zero resulting in \ILC{} achieving near-optimal cost despite having
non-zero modeling errors $\epsA, \epsB$. The suboptimality in \ILC{}
(from Lemma~\ref{lemma:ilc}) only
arises from $\epsA, \epsB$ and not from $\fILC_{t+1}$ which is
$0$.
In contrast, the upper
bound for \MM{}~\eqref{eq:14} does not go down to zero and has terms that
depend on $\epsA^2$, which can be significant when $\epsA$ is not
small. Thus, for highly damped systems we have that the worst case
performance of \ILC{} can be significantly better than \MM{}, especially
when $\epsA$ is large. Intuitively, this can be understood by
observing that \ILC{} removes 
the effect of modeling errors by always performing rollouts using true
dynamics, 
while \MM{} errors are exacerbated by using the approximate
model for rollouts.
Interestingly, we also notice that the modeling
error $\epsB$ does not affect the cost-suboptimality in
upper bound for \MM{}~\eqref{eq:14}  when the system is highly damped.

\paragraph{Weakly Controlled Systems}
\label{sec:small-b_t}

For systems with small $\|B_t\| << 1$, i.e. where the control inputs do not
affect the dynamics of the system to a large extent, we can observe
that the upper bound for \ILC{}~\eqref{eq:15} reduces to a bound that
does not depend $\epsB$. In other words, any modeling error $\epsB$ in
estimating the $B_t$ matrices does not affect the upper
bound~\eqref{eq:15} for \ILC{}.
In constrast, the upper bound for \MM{}~\eqref{eq:14} reduces to an
expression that has terms that depend
on $\epsB^2$, which can become significant when $\epsB$ is
large. Thus, for systems with $\|B_t\| << 1$, \ILC{} is robust to any
modeling errors $\epsB$ in the $B_t$ matrices, whereas \MM{} degrades its
worst case performance with increasing $\epsB$.

\paragraph{Modeling Error only at the first time step}
\label{sec:modeling-error-only}

Consider a scenario where the model is inaccurate only at $t=0$,
i.e. $\|A_0 - \Ahat_0\|\leq\epsA$ and 
$\|B_0 - \Bhat_0\|\leq\epsB$, while $\Ahat_t = A_t$ and $\Bhat_t =
B_t$ for all $t=1, \cdots, H-1$. In this case, the upper
bounds~\eqref{eq:14} and~\eqref{eq:15} simplify greatly as $\|\POPT_t
- \PILC_t\| = \|\POPT_t - \PCE_t\| = 0$ for all $t=1, \cdots, H-1$,
and we only have upper bounds on $||\POPT_0 - \PCE_0||$ and $||\POPT_0
- \PILC_0||$ as given by Theorems~\ref{theorem:ce}
and~\ref{theorem:ilc} which when combined with
Theorem~\ref{theorem:cost} gives us the suboptimality bounds:
\begin{align}
  \Vhat_0^{\mathsf{MM}}(x_0) - V_0^\star(x_0) &\leq
  \mathcal{O}(1)d\Gamma^9\|x_0\|^2(\epsA + \epsA^2 + \epsB +
                                                \epsB^2)^2 \label{eq:ce-cost}\\
  \Vhat_0^{\mathsf{ILC}}(x_0) - V_0^\star(x_0) &\leq
  \mathcal{O}(1)d\Gamma^9\|x_0\|^2(\epsA + \epsB)^2\label{eq:ilc-cost}                                                
\end{align}
The above two cost suboptimality bounds highlight the differences
between \MM{} and \ILC{} in worst case performance. As described in
Section~\ref{sec:small-model-errors}, if $\epsA$ and $\epsB$ are
small, then \MM{} and \ILC{} worst case performances match up to constants as
we can ignore higher order terms.
However, in cases where modeling errors $\epsA$ and $\epsB$ are large
and higher order terms like $\epsA^2\epsB$, $\epsA^4$ etc. start
becoming significant, the worst case performance of \ILC{} tends to be
better than \MM{} as indicated by equations~\eqref{eq:ce-cost}
and~\eqref{eq:ilc-cost}. Furthermore, the conditions for stability
under synthesized control inputs, as
stated in Theorem~\ref{theorem:cost} (and in
Lemma~\ref{lemma:stability} in Appendix~\ref{sec:general-results-1},) is harder to satisfy for \MM{}
when compared to \ILC{}, especially when modeling errors are large.

\section{Empirical Results}
\label{sec:empirical-results}

In this section, we present three empirical
experiments: a linear dynamical system with an approximate model, a
nonlinear inverted pendulum system with misspecified mass, and a nonlinear planar
quadrotor system in the presence of wind. The aim of these
experiments is to show that under high modeling errors, \ILC{} is
more efficient than \MM{}, thus backing our theoretical findings.\footnote{The
  code for all experiments can be found at
\url{https://github.com/vvanirudh/ILC.jl}.}

\subsection{Linear Dynamical System with Approximate Model}
\label{sec:line-dynam-syst}
In this experiment, we use a linear dynamical system with states
$x \in \reals^{2}$ and control inputs $u \in \reals$.
The dynamics of the system
are specified by matrices:
$A_{t} =
      \begin{bmatrix}
        1 & 1 \\
        -3 & 1
      \end{bmatrix},
  B_{t} =
          \begin{bmatrix}
            1 \\
            3
          \end{bmatrix}$.
The approximate model we use is constructed by perturbing the dynamics as
follows: $\Ahat_t = A_t + \epsilon\mathbb{I}$, $\Bhat_t = B_t +
\epsilon\begin{bmatrix}1 \\ 0\end{bmatrix}$
for any $\epsilon \geq 0$. Observe that this satisfies
$||\Ahat_{t} - A_{t}|| \leq \epsilon$ and $||\Bhat_{t} - B_{t}|| \leq \epsilon$.
We use a quadratic cost as specified in equation~\ref{eq:7} with matrices:
  $Q = Q_{f} = \mathbb{I}, R = 1$ (more details in Appendix~\ref{sec:line-dynam-syst-1}).
We can solve for the optimal controller $\KOPT$ in closed form using true
dynamics $A_{t}, B_{t}$ as specified in Section~\ref{sec:problem-setup}. We compare
\MM{} controller $\KCE$ 
and iterative learning controller $\KILC$ with approximate model
$\Ahat_{t}, \Bhat_{t}$ in \Cref{fig:lds} where we vary $\epsilon$ along
the X-axis (in log scale) and report the cost suboptimality gap
$V_{0}(x_{0}) - \VOPT_{0}(x_{0})$ on the Y-axis
(in log scale) where $V_{0}(x_{0})$ is the cost incurred by $\KCE$ or $\KILC$.
To ensure that
Assumption~\ref{assumption:psd} is not violated, the X-axis is capped at
$\epsilon = \frac{\ubar{\sigma}(B_{t}^{T}RB_{t})}{||B_{t}^{T}R||}$.
It is important to note that to generate the plot in \Cref{fig:lds} we
directly used the closed form solution for $\KILC$ (as described in
Section~\ref{sec:problem-setup}) and did not run a iterative learning control
algorithm. This was done to ensure that our results do not have any dependence
on how well the step size sequence was tuned for \ILC{}.

\begin{figure}[t]
  \centering
  \subfigure{\label{fig:lds}\includegraphics[width=.32\linewidth]{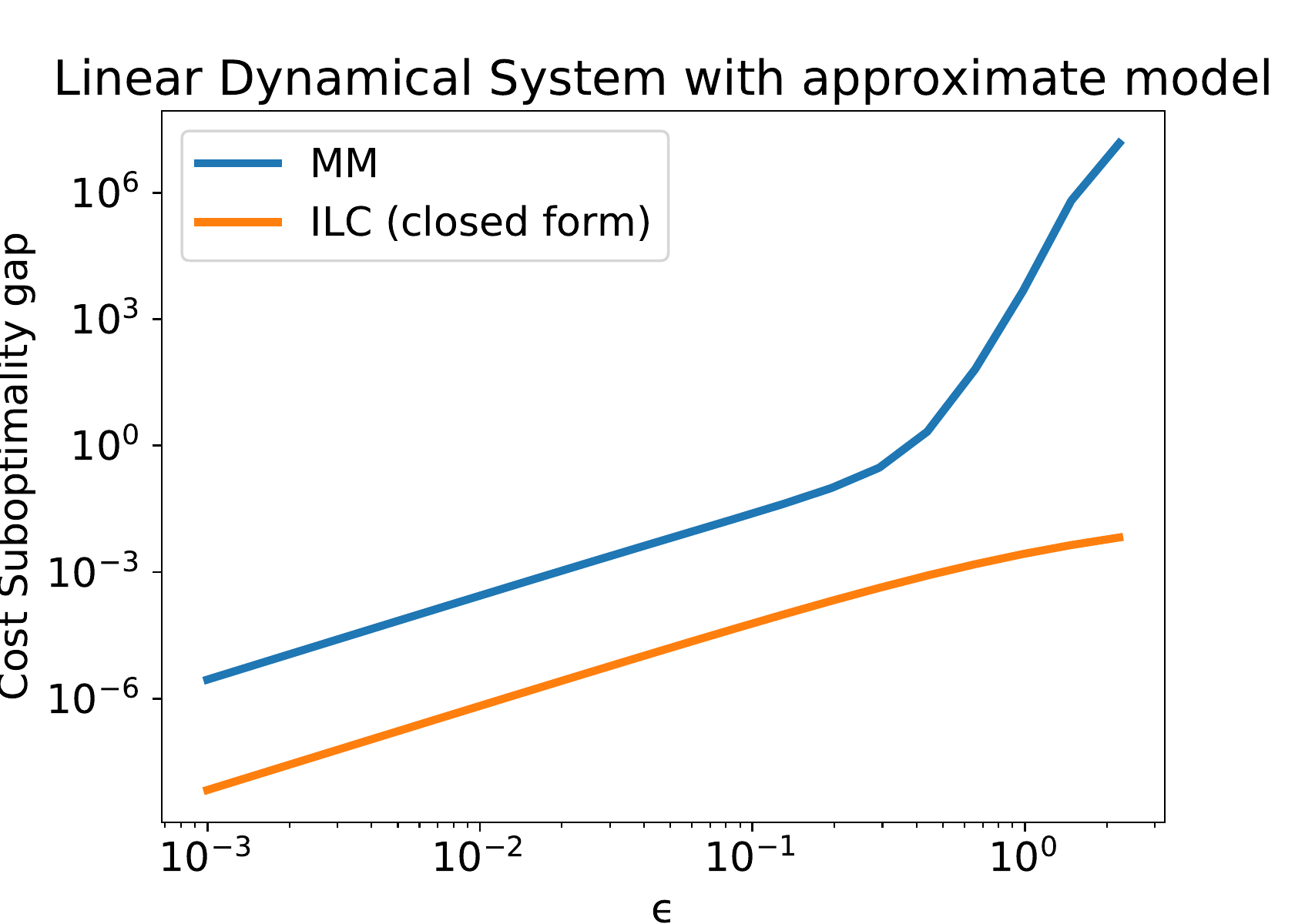}}
  \subfigure{\label{fig:pendulum}\includegraphics[width=.32\linewidth]{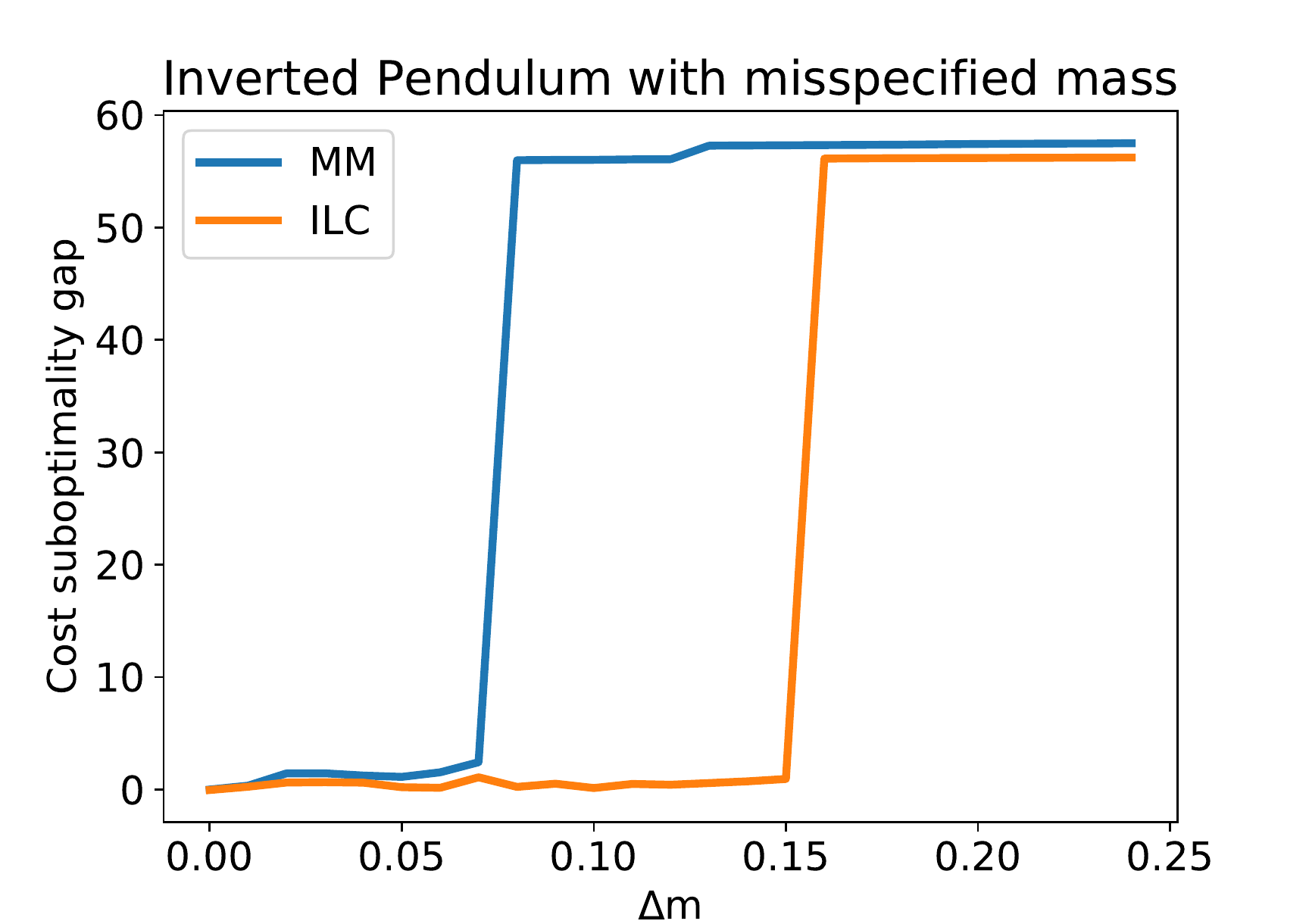}}
  \subfigure{\label{fig:quadrotor}\includegraphics[width=.32\linewidth]{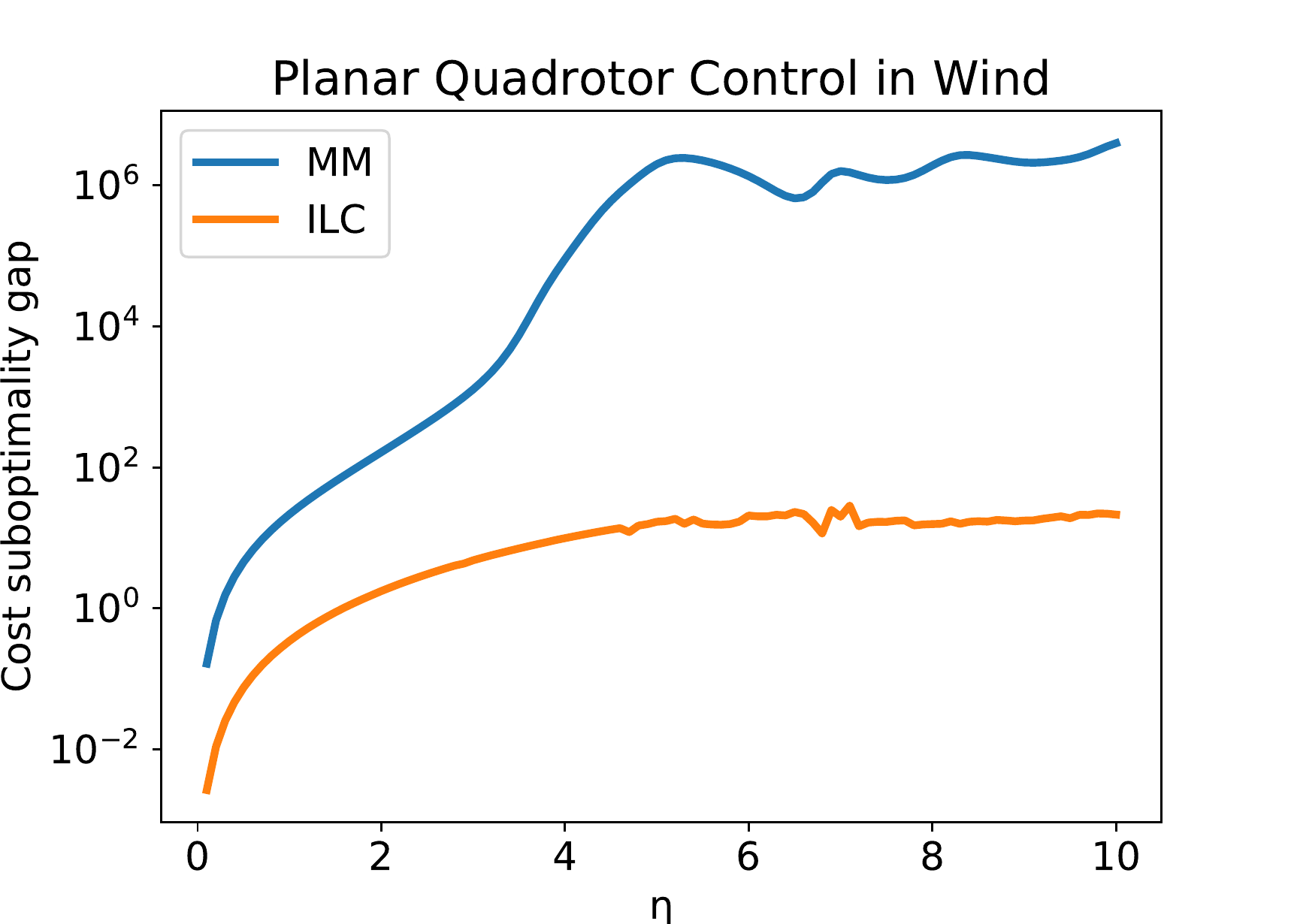}}
  \caption{(a) Cost suboptimality gap with varying
    modeling error $\epsilon$ for a
    linear dynamical system. Note that both X-axis
    and Y-axis are in log scale. (b) Cost suboptimality gap with
    varying mass misspecification $\Delta m$ for a nonlinear 
    inverted pendulum system. (c) Cost suboptimality gap for planar
    quadrotor control with varying magnitude of wind $\eta$.}
  \vspace{-0.3cm}
\end{figure}

We can observe that for small modeling errors $\epsilon < 10^{-1}$, \ILC{}
outperforms \MM{} by a constant factor (about $4\times 10^{2}$) as evidenced by the linear
trend in log scale. However in the regime of high modeling errors
$\epsilon > 10^{-1}$ we observe that the gap between \ILC{} and \MM{} is not a
constant factor anymore and grows very quickly as $\epsilon$ increases. This can
be explained by the fact that for high $\epsilon$, the
higher order terms in the gap between \ILC{} and \MM{} starts becoming significant
and results in poor performance for \MM{} when compared to \ILC{}. For large epsilons,
we also observe that the cost for \MM{} blows up to really big values as the system
is not stable anymore under $\KCE$ due to violation of the condition
in Theorem~\ref{theorem:cost} (and in
Lemma~\ref{lemma:stability}.) This experiment validates our claim from the
analysis that \ILC{} tends to perform better in terms of cost and is more robust
when modeling errors are high.

\subsection{Nonlinear Inverted Pendulum with Misspecified Mass}
\label{sec:invert-pend-with}

For the second experiment, we use the nonlinear dynamical system of an inverted
pendulum. The state space is specified by $x =
\begin{bmatrix}
  \theta &
  \dot{\theta}
\end{bmatrix} \in \reals^{2}
$ where $\theta$ is the angle between the pendulum and the vertical axis. The
control input is $u = \tau \in \reals$ specifying the torque $\tau$ to be
applied at the base of the pendulum. The dynamics of the system are
given by the ODE,
$\ddot{\theta} = \frac{\bar{\tau}}{m\ell^{2}} - \frac{g\sin(\theta)}{\ell}$
where $m$ is the mass of the pendulum, $\ell$ is the length of the pendulum, $g$
is the acceleration due to gravity, and
$\bar{\tau} = \max(\tau_{\min}, \min(\tau_{\max}, \tau))$ is the clipped torque
based on torque limits (more details in
Appendix~\ref{sec:nonl-invert-pend}). 
We use an approximate model of the dynamics where the mass of the pendulum is
perturbed as $\hat{m} = m + \Delta m$.
This results in dynamics that are nonlinearly perturbed from the true dynamics.
Since the dynamics are nonlinear, we cannot obtain
optimal controls, and \MM{} controls in closed form. Instead, we approximate these
controllers by running iLQR~\cite{li04} (both forward and backward pass) on the
true dynamics and the approximate
dynamics respectively for $200$ iterations. To obtain \ILC{} control inputs, we run
iLQR with forward pass (or rollouts)
using the true dynamics, and backward pass computed using the approximate
dynamics at each iteration.
We chose step sizes for all iLQR runs using
backtracking line search.


\Cref{fig:pendulum} shows the cost suboptimality gap of \MM{} and \ILC{} as the
perturbation $\Delta m$ varies. Similar to our previous experiment, we observe
that for small modeling errors $\Delta m < 0.07$ both \ILC{} and \MM{} perform
similarly with \ILC{} outperforming slightly. But as $\Delta m$ grows, the cost of
\MM{} quickly grows saturating at a suboptimality gap around $57$.
In contrast, we observe that \ILC{}
is still able to compute near-optimal controls until
$\Delta m = 0.15$ showcasing the robustness of \ILC{} to higher modeling errors.
Beyond $\Delta m = 0.15$, \ILC{} performance also degrades significantly as the
approximate model is not representative of the true dynamics anymore.
Although our analysis in the previous sections was restricted to linear
dynamical systems, we notice a similar trend between \ILC{} and \MM{} in the presence
of nonlinear dynamics namely, in the regime of large modeling errors, \ILC{} tends to
perform better than \MM{}.

\subsection{Nonlinear Planar Quadrotor Control in Wind}
\label{sec:nonl-plan-quadr}

In our final experiment, we compare \MM{} and \ILC{} on a planar quadrotor control
task in the presence of wind. A similar setting was used
in~\cite{DBLP:conf/icml/AgarwalHMS21}. The quadrotor is controlled using two propellers
that provide upward
thrusts $(u_{1}, u_{2})$ and allows movement in the $3$D planar space
described as
$(p_{x}, p_{y}, \theta)$ where $p_{x}, p_{y}$ are X, Y positions, and
$\theta$ is the yaw of the quadrotor. The dynamics of the planar quadrotor
is specified
using a state vector $x \in \reals^{6}$, and control input $u \in
\reals^{2}$ (more details in Appendix~\ref{sec:nonl-plan-quadr-1}).
The quadrotor is flying in the presence of wind which is not captured in modeled
dynamics, but affects the true dynamics of the quadrotor as a dispersive force
field $(\eta p_{x}\mathbf{i} + \eta p_{y} \mathbf{j})$ resulting in
overall dynamics given by:
\begin{align*}
  \ddot{p}_{x} = \frac{1}{m}(u_{1} + u_{2})\sin(\theta) + \eta p_{x}~~~~~~~
  \ddot{p}_{y} = \frac{1}{m}(u_{1} + u_{2})\cos(\theta) - g + \eta p_{y}
\end{align*}
where $\eta \in \reals^{+}$ is a constant that captures magnitude of the wind
force field.

The objective of the task is to move the quadrotor from an initial
state $x_0$
to a final state $x_f$. Similar to
previous experiment, the dynamics are nonlinear and we cannot obtain optimal
controls and
\MM{} controls in closed form. Thus, we again approximate these by running
iLQR on true dynamics and approximate dynamics respectively. We obtain \ILC{}
control inputs again by using iLQR with forward pass using true dynamics and
backward pass using approximate dynamics.
For all iLQR runs, we choose step sizes by performing backtracking line search
and we initialize the control inputs as the hover controls.
\Cref{fig:quadrotor} compares \MM{} and \ILC{} for planar quadratic control with
varying magnitude of wind $\eta$. For small wind magnitudes, we
observe that both \MM{} and \ILC{} have good performance. As the wind
magnitude increases, \MM{} quickly diverges and the cost of synthesized
control inputs blows up quickly as the modeled dynamics are incapable
of capturing the dispersive force field exerted by the wind. \ILC{}, on
the other hand, manages to keep the cost from blowing up even at large
wind magnitudes. This reinforces our conclusion that \ILC{} is robust to
large modeling errors while \MM{} can quickly result in the cost blowing
up when the model is highly inaccurate.

\section{Discussion}
\label{sec:discussion}


Our analysis shows that the gap between \ILC{} and \MM{} is in higher
order terms that can become
significant when the modeling error $\epsA, \epsB$ is large. This is
backed by our empirical experiments where we observe that as the
magnitude of modeling error increases, the performance gap between \ILC{}
and \MM{} grows rapidly as \MM{} is incapable of handling large modeling
errors and the resulting cost diverges. Furthermore, the conditions
needed for stability of the system under synthesized control inputs,
are easier to satisfy for \ILC{} when compared to \MM{}, especially in the
regime of large modeling errors. This explains the robustness of \ILC{}
over \MM{} for complex control tasks when given access to highly
inaccurate dynamical models. We also identify scenarios where the
norms $\|A_t\|$ and $\|B_t\|$ are small, where \ILC{} is
provably more efficient and more robust to modeling errors, when
compared to \MM{}.

While our current analysis is restricted to the linear quadratic
control setting, exploring similar suboptimality bounds in more
complex and possibly, nonlinear settings is an exciting direction for
future work. Recent work by~\cite{DBLP:conf/icml/SimchowitzF20} uses a
self-bounding ODE method to establish perturbation bounds that
sharpens previous bounds in the infinite horizon setting by only depending on natural
control-theoretic quantities and not relying on controllability
assumptions. It remains to be seen if we can rely on similar
techniques to sharpen the bounds presented in this work.
It would also be interesting to know whether fast rates for
control are possible for cost functions other than quadratic
costs. Finally, comparing iterative learning control and robust
control approaches such as \cite{dean20} would allow us to understand
the regime of modeling errors in which \ILC{} is more suitable than
robust control approaches, and vice versa.

\acks{AV would like to thank Horia Mania for helping with
  understanding some technical lemmas in \cite{mania19}. AV is
  supported by the CMU Presidential Fellowship endowed by TCS.} 

\bibliography{ref}

\appendix

\section{General Results}
\label{sec:general-results-1}

In this section, we will present general results that bound the cost suboptimality of
any time-varying controller $\Khat$ in terms of the norm differences
$||\KOPT_{t} - \Khat_{t}||$. Our first lemma makes use of
Assumption~\ref{assumption:stability} to show that if the norm differences
$||\KOPT_{t} - \Khat_{t}||$ are small, then the true system can be stable under
$\Khat$:
\begin{restatable}{lemma}{stabilityLemma}
  \label{lemma:stability}
  If Assumption~\ref{assumption:stability} holds and if $\Khat$ satisfies
  $||\KOPT_{i} - \Khat_{i}|| \le \frac{\delta}{2||B_{i}||}$ for all
  $i \in \{0, \cdots, H-1\}$, then we have
  \begin{equation}
    \label{eq:1}
    ||L_{t}(\Khat)|| \leq \left(1 - \frac{\delta}{2}\right)^{t+1} \leq e^{-\frac{\delta}{2}(t+1)}
  \end{equation}
\end{restatable}
\begin{proof}
  Observe that,
\begin{align*}
  ||L_t(\Khat)|| &= ||\prod_{i=0}^t M_i(\Khat)|| = ||\prod_{i=0}^t A_i
                   + B_i\Khat_i|| \\
  &= ||\prod_{i=0}^t A_i + B_i\KOPT_i + B_i(\Khat_i - \KOPT_i)|| =
    ||\prod_{i=0}^t M_i(\KOPT) + \Delta_i||
\end{align*}
where $\Delta_i = B_i(\Khat_i - \KOPT_i)$. Since the spectral norm is
sub-multiplicative we can see that
\begin{align*}
  ||\prod_{i=0}^t M_i(\KOPT) + \Delta_i|| &\leq \prod_{i=0}^t||M_i(\KOPT) +
                                        \Delta_i|| \\
  &\leq \prod_{i=0}^t(||M_i(\KOPT)|| + ||\Delta_i||)
\end{align*}
where we used the triangle inequality. Now note that
$||M_i(\KOPT)|| \leq 1-\delta$ from assumption~\ref{assumption:stability},
\begin{align*}
  \|\Delta_{i}\| &= \|B_{i}(\Khat_{i} - \KOPT_{i}) \| \leq \|B_{i}\|\|\Khat_{i} - \KOPT_{i}\| \\
                 &\leq \kappa \frac{\delta}{2\kappa} \\
  &\leq \frac{\delta}{2}
\end{align*}
The last inequality above is from our assumption on model errors in the lemma statement.
Combining all of this above, we get
\begin{align*}
  ||L_t(\Khat)|| \leq \prod_{i=0}^t\left(1 - \delta + \frac{\delta}{2}\right) \leq \left( 1-\frac{\delta}{2} \right)^{t+1}
\end{align*}
\end{proof}

The next lemma is very similar to the performance difference lemma that was
first proposed in~\cite{kakade02}. We borrow the version presented
in~\cite{fazel18} and extend it to the finite horizon setting below:
\begin{lemma}
  \label{lemma:performance-difference}
  Let $\xhat_0,
  \uhat_0, \cdots, \xhat_H, \uhat_H$ be the trajectory generated by
  controller
  $\Khat$ using the true dynamics such that $\xhat_0 = x_0$, $\uhat_t = \Khat_t\xhat_t$ for
  $t=0, \cdots, H-1$. Then
  we have:
  \begin{equation}
    \label{eq:2}
    \Vhat_0(x_0) - \VOPT_0(x_0) = \sum_{t=0}^{H-1} \AOPT_t(\xhat_t, \uhat_t) - \VOPT_H(\xhat_H)
  \end{equation}
  where $\Vhat_t$ is the cost-to-go using controller $\Khat$ from time
  step $t$, $\VOPT_t$ is the cost-to-go using the optimal controller $\KOPT$ from time
  step $t$, and $\AOPT_t(x, u) = \QOPT_t(x, u) - \VOPT_t(x)$ is the advantage of
  the controller $\KOPT$ at time step $t$. Furthermore, we have that for any $x$
  \begin{equation*}
    \AOPT_t(x, \Khat_tx) = x^T(\Khat_t - \KOPT_t)^T(R + B_t^T\POPT_{t+1}B_t)
    (\Khat_t - \KOPT_t)x
  \end{equation*}
\end{lemma}
\begin{proof}
  For proof, we refer the readers to~\cite{fazel18}.
\end{proof}

We use the performance difference lemma, as stated above, in the finite horizon
LQR setup and make use of Lemma~\ref{lemma:stability} to establish the
suboptimality bound in terms of the norm differences $||\KOPT_{t} - \Khat_{t}||$:
\costTheorem*
\begin{proof}
  From Lemma~\ref{lemma:performance-difference} we have
  \begin{align*}
    A_{t}(\xhat_{t}, \Khat_{t}\xhat_{t}) &= \xhat_{t}^{T}(\Khat_{t} - \KOPT_{t})^{T}(R + B_{t}^{T}P_{t+1}B_{t})(\Khat_{t} - \KOPT_{t})\xhat_{t}
  \end{align*}
  We know $\xhat_{t} = L_{t-1}(\Khat)x_{0}$ and using the trace identity we get
  \begin{align*}
    A_{t}(\xhat_{t}, \Khat_{t}\xhat_{t}) &= \Tr(L_{t-1}(\Khat)x_{0}x_{0}^{T}(L_{t-1}(\Khat))^{T}(\Khat_{t} - \KOPT_{t})^{T}(R + B_{t}^{T}P_{t+1}B_{t})(\Khat_{t} - \KOPT_{t})) \\
                                         &\leq \|L_{t-1}(\Khat)x_{0}x_{0}^{T}\| \|R + B_{t}^{T}P_{t+1}B_{t}\| \|\Khat_{t} - \KOPT_{t}\|_{F}^{2} \\
    &\leq \|L_{t-1}(\Khat)\|^{2}\|x_{0}\|^{2} \|R + B_{t}^{T}P_{t+1}B_{t}\| \|\Khat_{t} - \KOPT_{t}\|_{F}^{2}
  \end{align*}
  We can bound $\|R + B_{t}^{T}P_{t+1}B_{t}\| \leq \Gamma^{3}$ and
  $\|\Khat_{t} - \KOPT_{t}\|_{F}^{2} \leq \min\{n, d\} \|\Khat_{t} - \KOPT_{t}\|^{2}$.
  We can also use Lemma~\ref{lemma:stability} to bound
  $\|L_{{t-1}}(\Khat)\| \leq \exp\left(-\frac{\delta}{2}t\right)$.
  Combining all
  of this above we get
  \begin{align*}
    A_{t}(\xhat_{t}, \Khat_{t}\xhat_{t}) &\leq \min\{n, d\}\Gamma^{3} \exp\left(-\delta t\right)\|\Khat_{t} - \KOPT_{t}\|^{2}\|x_{0}\|^{2}
  \end{align*}
  Summing over all time steps we
  obtain (using $d \leq n$)
  \begin{align*}
    \Vhat_0(x_0) - V_0(x_0) &\leq d\Gamma^{3}\|x_{0}\|^{2} \sum_{t=0}^{H-1} \exp\left(-\delta t\right)\|\Khat_{t} - \KOPT_{t}\|^{2}
  \end{align*}
\end{proof}

\section{Helpful Lemmas}
\label{sec:helpful-lemmas}

Before we dive into the results, let us present a helpful lemma
borrowed from \cite{mania19}:
\begin{lemma}
  Let $f_1, f_2$ be $\mu$-strongly convex twice differentiable
  functions. Let $x_1 = \arg\min_x f_1(x)$ and $x_2 = \arg\min_x
  f_2(x)$. Suppose $||\nabla f_1(x_2)|| \leq \epsilon$, then $||x_1 -
  x_2|| \leq \frac{\epsilon}{\mu}$
  \label{lemma:1}
\end{lemma}
\begin{proof}
  Taylor expanding $\nabla f_{1}$ we get
  \begin{align*}
    \nabla f_{1}(x_{2}) &= \nabla f_{1}(x_{1}) + \nabla^{2}f_{1}(\tilde{x})(x_{2} - x_{1}) \\
    &= \nabla^{2}f_{1}(\tilde{x})(x_{2} - x_{1})
  \end{align*}
  for some $\tilde{x} = tx_{1}+ (1-t)x_{2}$ where $t \in [0, 1]$. Thus we have
  \begin{align*}
    \|\nabla f_{1}(x_{2})\| = \|\nabla^{2}f_{1}(\tilde{x})(x_{2} - x_{1})\| \leq \epsilon
  \end{align*}
  But we know $\|\nabla^{2}f_{1}(\tilde{x})\| \geq \mu$ which gives us
  \begin{align*}
    \|x_{2} - x_{1}\| \leq \frac{\epsilon}{\mu}
  \end{align*}
\end{proof}

The next lemma is a useful fact about positive semi-definite matrices,
also from \cite{mania19},
\begin{lemma}
  \label{lemma:mania-appendix-original}
  Given matrices $A, \Ahat$ such that $\|A - \Ahat\| \leq \epsA$,
  and positive-semidefinite matrices $Q, S, \Shat$ we have
  \begin{equation}
    \label{eq:89}
    \|A^{T}Q(I + SQ)^{-1}A - \Ahat^{T}Q(I + \Shat Q)^{-1}\Ahat\| \leq \|A\|^{2}\|Q\|^{2}\|\Shat - S\| + 2\|A\|\|Q\|\epsA + \|Q\|\epsA^{2}
  \end{equation}
\end{lemma}
\begin{proof}
  We can rewrite the expression,
  \begin{align*}
    A^{T}Q(I + SQ)^{-1}A &- \Ahat^{T}Q(I + \Shat Q)^{-1}\Ahat = \\
                         &A^{T}Q(I + SQ)^{-1}(\Shat - S)Q(I+\Shat Q)^{-1}A - A^{T}Q(I + \Shat Q)^{-1}(\Ahat - A) \\
    &-(\Ahat - A)^{T}Q(I + \Shat Q)^{-1}A - (\Ahat - A)^{T}Q(I + \Shat Q)^{-1}(\Ahat - A)
  \end{align*}

  Now we make use of Lemma 7 from \cite{mania19} which states that for any two
  positive semidefinite matrices $M, N$ of the same dimension, we have
  $\|N(I + MN)^{-1}\| \leq \|N\|$. Thus, we have $\|Q(I+SQ)^{-1}\| \leq \|Q\|$
  and $\|Q(I + \Shat Q)^{-1}\| \leq \|Q\|$.

  Using the above facts we get,
  \begin{align*}
    \|A^{T}Q(I + SQ)^{-1}A - \Ahat^{T}Q(I + \Shat Q)^{-1}\Ahat\| &\leq \|A\|^{2}\|Q\|^{2}\|\Shat - S\| + 2\|A\|\|Q\|\epsA + \|Q\|\epsA^{2}
  \end{align*}
\end{proof}

Finally, we have a lemma that will be useful in proving ricatti
perturbation bounds,
\begin{lemma}
  Given positive semidefinite matrices $N_{1}, N_{2}, M$ of the same dimensions,
  we have
  \begin{equation}
    \label{eq:90}
    ||N_{1}(I + MN_{1})^{-1} - N_{2}(I + MN_{2})^{-1}|| \leq ||(I + MN_{1})^{-1}|| ||N_{1} - N_{2}|| ||(I + MN_{2})^{-1}||
  \end{equation}
  \label{lemma:stackexchange}
\end{lemma}
\begin{proof}
  We can rewrite the expression as,
  \begin{align*}
    &N_1(I+MN_1)^{-1}-N_2(I+MN_2)^{-1}\\
    &=\left[N_1(I+MN_1)^{-1}-N_1(I+MN_2)^{-1}\right]
    +\left[N_1(I+MN_2)^{-1}-N_2(I+MN_2)^{-1}\right]\\
    &=N_1(I+MN_1)^{-1}\left[(I+MN_2)-(I+MN_1)\right](I+MN_2)^{-1}
    +(N_1-N_2)(I+MN_2)^{-1}\\
    &=N_1(I+MN_1)^{-1}M(N_2-N_1)(I+MN_2)^{-1}
    +(N_1-N_2)(I+MN_2)^{-1}\\
    &=\left[I-N_1(I+MN_1)^{-1}M\right](N_1-N_2)(I+MN_2)^{-1} \\
    &=(I + N_{1}M)^{{-1}}(N_1-N_2)(I+MN_2)^{-1}
  \end{align*}
  The rest follows by taking norm on both sides, and using the submultiplicative
  property of the induced norm.
  \begin{align*}
    ||N_{1}(I + MN_{1})^{-1} - N_{2}(I + MN_{2})^{-1}|| &\leq ||(I + N_{1}M)^{-1}|| ||N_{1} - N_{2}|| ||(I + MN_{2})^{-1}|| \\
                                                        &= ||(I + N_{1}M)^{-T}||||N_{1} - N_{2}|| ||(I + MN_{2})^{-1}|| \\
    &= ||(I + MN_{1})^{-1}||||N_{1} - N_{2}|| ||(I + MN_{2})^{-1}||
  \end{align*}
\end{proof}

\section{Optimal Control with Misspecified Model Results}
\label{sec:cert-equiv-contr-2}

The next lemma, from \cite{mania19}, applies the above result to quadratic functions that
are observed in linear quadratic control:
\begin{lemma}
  \label{lemma:ce-quadratic}
  Define $f_1(x, u) = \frac{1}{2}u^TRu +
  \frac{1}{2}(A_1x+B_1u)^TP_1(A_1x+B_1u)$ and similarly define $f_2(x,
  u)$ where $R, P_1, P_2$ are positive-definite matrices. Let $K_1$ be
  such that $u_1 = \arg\min_u f_1(x, u) = K_1x$  for
  any vector $x$. Define the matrix $K_2$ in a similar fashion. Also,
  denote $\Gamma = 1 + \max\{||A_1||, \allowbreak||B_1||, ||P_1||,
  ||K_1||\}$. Suppose there exists $\epsA, \epsB, \epsP > 0$ (and
  $<\Gamma$) such that 
  $||A_1 - A_2|| \leq \epsA$, $||B_1 - B_2|| \leq
  \epsB$, and $||P_1 - P_2|| \leq \epsP$. Then we have,
  \begin{equation}
    \label{eq:73}
    \|K_{1} - K_{2}\| \leq \frac{\Gamma^2\epsA + (3\Gamma^3 +
      2\Gamma^2)\epsB + 4(\Gamma^3 + \Gamma^2)\epsP}{\ubar{\sigma}(R)}
  \end{equation}
\end{lemma}

\begin{proof}
  Consider
  \begin{align*}
    \nabla_{u} f_{1}(x, u) &= (B_{1}^{T}P_{1}B_{1} + R)u + B_{1}^{T}P_{1}A_{1}x \\
    \nabla_{u} f_{2}(x, u) &= (B_{2}^{T}P_{2}B_{2} + R)u + B_{2}^{T}P_{2}A_{2}x
  \end{align*}

  Let us bound the difference $\|\nabla_{u} f_{1}(x,u) - \nabla_{u} f_{2}(x, u)\|$
  by bounding each term separately. First consider the term
  \begin{align*}
    \|B_{1}^{T}P_{1}B_{1} - B_{2}^{T}P_{2}B_{2}\| &= \|B_{1}^{T}P_{1}(B_{1} - B_{2}) + (B_{1} - B_{2})^{T}P_{1}B_{2} + B_{2}^{T}(P_{1} - P_{2})B_{2}\| \\
                                                  &\leq \|B_{1}^{T}P_{1}(B_{1} - B_{2})\| + \|(B_{1} - B_{2})^{T}P_{1}B_{2}\| + \|B_{2}^{T}(P_{1} - P_{2})B_{2}\| \\
                                                  &\leq \Gamma^{2}\epsB + \Gamma\epsB(\Gamma + \epsB) + (\Gamma + \epsB)^{2}\epsP \\
    &\leq \Gamma^2(3\epsB + 4\epsP)
  \end{align*}
  where we used the fact that $\|B_{2}\| \leq \Gamma + \epsB$.
  We can similarly bound the term
  \begin{align*}
    \|B_{1}^{T}P_{1}A_{1} - B_{2}^{T}P_{2}A_{2}\| \leq
    \Gamma^{2}(\epsA + 2\epsB + 4\epsP)
  \end{align*}

  Thus, we have for any vector $x$ such that $\|x\| \leq 1$
  \begin{align*}
    \|\nabla_{u} f_{1}(x,u) - \nabla_{u} f_{2}(x, u)\| \leq
    \Gamma^2(3\epsB + 4\epsP)\|u\| + \Gamma^2(\epsA + 2\epsB + 4\epsP)
  \end{align*}
  Substituting $u=u_{1}$ we get
  \begin{align*}
    \|\nabla_{u} f_{2}(x, u_{1})\| \leq \Gamma^2(3\epsB + 4\epsP)\|u_1\| + \Gamma^2(\epsA + 2\epsB + 4\epsP)
  \end{align*}

  We can bound $\|u_{1}\| \leq \|K_{1}\|\|x\| \leq \|K_{1}\| \leq \Gamma$. Then
  from Lemma~\ref{lemma:1} we have,
  \begin{align*}
    \|u_{1} - u_{2}\| &\leq \frac{\Gamma^3(3\epsB + 4\epsP) +
                        \Gamma^2(\epsA + 2\epsB + 4\epsP)}{\ubar{\sigma}(R)} \\
    \|K_{1} - K_{2}\| &\leq \frac{\Gamma^2\epsA + (3\Gamma^3 +
      2\Gamma^2)\epsB + 4(\Gamma^3 + \Gamma^2)\epsP}{\ubar{\sigma}(R)}
  \end{align*}
\end{proof}

Now we will prove Lemma~\ref{lemma:ce},
\ceLemma*
\begin{proof}
  Use Assumption~\ref{assumption:singularvalue} and
  Lemma~\ref{lemma:ce-quadratic} for every $t=0, \cdots, H-1$ with 
  $\epsP = \fCE_{t+1}(\epsA, \epsB)$ and choosing $\epsilon_t =
  \max\{\epsA, \epsB, \fCE_{t+1}(\epsA, \epsB)\}$.
\end{proof}

All that is left is to prove Theorem~\ref{theorem:ce} which we will do
now,
\theoremCE*
\begin{proof}
  We know $\POPT_{t}$ satisfies,
\begin{align*}
  \POPT_{t} &= Q + A_t^{T}\POPT_{{t+1}}A_t - A_t^{T}\POPT_{t+1}B_t(R + B_t^{T}\POPT_{{t+1}}B_t)^{-1}B_t^{T}\POPT_{t+1}A_t \\
  &= Q + A_t^{T}\POPT_{t+1}(I + B_tR^{-1}B_t^{T}\POPT_{t+1})^{-1}A_t
\end{align*}
where we used the matrix inversion lemma.

Similarly we have,
\begin{align*}
  \PCE_{t} &= Q + \Ahat_t^{T}\PCE_{t+1}(I + \Bhat_t R^{-1}\Bhat_t^{T}\PCE_{t+1})^{{-1}}\Ahat_t
\end{align*}

Consider the difference,
\begin{align*}
  \POPT_{t} - \PCE_{t} &= A_t^{T}\POPT_{t+1}(I + B_tR^{-1}B_t^{T}\POPT_{t+1})^{-1}A_t - \Ahat_t^{T}\PCE_{t+1}(I + \Bhat_t R^{-1}\Bhat_t^{T}\PCE_{t+1})^{{-1}}\Ahat_t \\
                    &= A_t^{T}\POPT_{t+1}(I + B_tR^{-1}B_t^{T}\POPT_{t+1})^{-1}A_t - \Ahat_t^{T}\POPT_{t+1}(I + \Bhat_t R^{-1}\Bhat_t^{T}\POPT_{t+1})^{{-1}}\Ahat_t \\
  &+ \Ahat_t^{T}\left(\POPT_{t+1}(I + \Bhat_t R^{-1}\Bhat_t^{T}\POPT_{t+1})^{{-1}} - \PCE_{t+1}(I + \Bhat_t R^{-1}\Bhat_t^{T}\PCE_{t+1})^{{-1}} \right)\Ahat_t
\end{align*}
To bound the above expression, we will make use of
Lemma~\ref{lemma:mania-appendix-original} with $S = B_tR^{-1}B_t^{T}$,
$\Shat = \Bhat_t R^{-1}\Bhat_t^{T}$, 
$Q = \POPT_{t+1}$ and observing that
$\|\Shat - S\| \leq 2\|B_t\|\|R^{-1}\|\epsB + \|R^{-1}\|\epsB^{2}$ we obtain

\begin{align*}
  \|\PCE_{t} - \POPT_{t}\| \leq& \|A_t\|^{2}\|\POPT_{t+1}\|^{2}(2\|B_t\|\|R^{-1}\|\epsB + \|R^{-1}\|\epsB^{2}) + 2\|A_t\|\|\POPT_{t+1}\|\epsA + \|\POPT_{t+1}\|\epsA^{2} \\
  &+ \|\Ahat_t^{T}\left(\POPT_{t+1}(I + \Bhat_t R^{-1}\Bhat_t^{T}\POPT_{t+1})^{{-1}} - \PCE_{t+1}(I + \Bhat_t R^{-1}\Bhat_t^{T}\PCE_{t+1})^{{-1}} \right)\Ahat_t\|
\end{align*}

All that remains is to bound the second expression. We will use
Lemma~\ref{lemma:stackexchange} with $N_{1} = \POPT_{t+1}$, $N_{2} = \PCE_{t+1}$ and
$M = \Bhat_t R^{{-1}}\Bhat_t^{T}$ gives us,
\begin{align*}
  &||\POPT_{t+1}(I + \Bhat_t R^{-1}\Bhat_t^{T}\POPT_{t+1})^{{-1}} - \PCE_{t+1}(I + \Bhat_t R^{-1}\Bhat_t^{T}\PCE_{t+1})^{{-1}}|| \\
  &\leq ||(I + \Bhat_t R^{{-1}}\Bhat_t^{T}\POPT_{t+1})^{-1}|| ||\POPT_{t+1} - \PCE_{{t+1}}|| ||(I + \Bhat_t R^{{-1}}\Bhat_t^{T}\PCE_{t+1})^{-1}||
\end{align*}

Thus, we have
\begin{align*}
  ||\POPT_{t} - \PCE_{t}|| &\leq  \|A_t\|^{2}\|\POPT_{t+1}\|^{2}(2\|B_t\|\|R^{-1}\|\epsB + \|R^{-1}\|\epsB^{2}) + 2\|A_t\|\|\POPT_{t+1}\|\epsA + \|\POPT_{t+1}\|\epsA^{2} \\
  &+ \|\Ahat_t\|^{2}||(I + \Bhat_t R^{{-1}}\Bhat_t^{T}\POPT_{t+1})^{-1}|| ||\POPT_{t+1} - \PCE_{{t+1}}|| ||(I + \Bhat_t R^{{-1}}\Bhat_t^{T}\PCE_{t+1})^{-1}||
\end{align*}

Observe that we can bound
\begin{align*}
  ||(I + \Bhat_t R^{{-1}}\Bhat_t^{T}\POPT_{t+1})^{-1}|| &= ||(\POPT_{t+1})^{-1}\POPT_{t+1}(I + \Bhat_t R^{{-1}}\Bhat_t^{T}\POPT_{t+1})^{-1}|| \\
                                                &\leq ||(\POPT_{t+1})^{-1}|| ||\POPT_{t+1}(I + \Bhat_t R^{{-1}}\Bhat_t^{T}\POPT_{t+1})^{-1}|| \\
  &\leq ||(\POPT_{t+1})^{-1}|| ||\POPT_{t+1}|| = \kappa_{\POPT_{t+1}}
\end{align*}
where $\kappa_{\POPT_{t+1}}$ is the condition number of the matrix $\POPT_{t+1}$. This
gives us the bound
\begin{align*}
  ||\POPT_{t} - \PCE_{t}|| &\leq  \|A_t\|^{2}\|\POPT_{t+1}\|^{2}(2\|B_t\|\|R^{-1}\|\epsB + \|R^{-1}\|\epsB^{2}) + 2\|A_t\|\|\POPT_{t+1}\|\epsA + \|\POPT_{t+1}\|\epsA^{2} \\
                        &+
                          \|\Ahat_t\|^{2}\kappa_{\POPT_{t+1}}\kappa_{\PCE_{t+1}}
                          ||\POPT_{t+1} - \PCE_{{t+1}}||                          
\end{align*}

Using the fact that
$||\Ahat_t||^{2} \leq (\|A_t\| + \epsA)^2$ gives us
\begin{align}
  \label{eq:92}
  ||\POPT_{t} - \PCE_{t}|| &\leq  \|A_t\|^{2}\|\POPT_{t+1}\|^{2}(2\|B_t\|\|R^{-1}\|\epsB + \|R^{-1}\|\epsB^{2}) + 2\|A_t\|\|\POPT_{t+1}\|\epsA + \|\POPT_{t+1}\|\epsA^{2} \nonumber\\
                        &+ (\|A_t\| +  \epsA)^2\kappa_{\POPT_{t+1}}\kappa_{\PCE_{t+1}} ||\POPT_{t+1} - \PCE_{{t+1}}||
\end{align}

If $\epsA, \epsB$ are small enough that $\|\POPT_{t+1} -
\PCE_{t+1}\|\|\POPT_{t+1}\| \leq 1$ then we can bound 
\begin{align*}
  \|(\PCE_{t+1})^{-1} - (\POPT_{t+1})^{-1}\| &\leq
  \frac{\|(\POPT_{t+1})^{-1}\|}{1 - \|(\POPT_{t+1})^{-1}\|\|\POPT_{t+1}
                                               - \PCE_{t+1}\|} \\
  &\leq \frac{\|(\POPT_{t+1})^{-1}\|\|\POPT_{t+1}\|}{\|\POPT_{t+1}\| -
    \|(\POPT_{t+1})^{-1}\|} \\
  &\leq \frac{\kappa_{\POPT_{t+1}}}{\|\POPT_{t+1}\|^2 - \kappa_{\POPT_{t+1}}}
\end{align*}
The above result is from~\cite{horn12} (Section 5.8 page 381). Now we
can bound the condition number $\kappa_{\PCE_{t+1}}$ by observing that
$\|(\PCE_{t+1})^{-1}\| \leq \|(\POPT_{t+1})^{-1}\| +
\frac{\kappa_{\POPT_{t+1}}}{\|\POPT_{t+1}\|^2 - \kappa_{\POPT_{t+1}}}$
and $\|\PCE_{t+1}\| \leq \|\POPT_{t+1}\| + \|\PCE_{t+1} -
\POPT_{t+1}\| \leq \|\POPT_{t+1}\| + \frac{1}{\|\POPT_{t+1}\|}$ giving
us
\begin{align*}
  \kappa_{\POPT_{t+1}}\kappa_{\PCE_{t+1}} &= \kappa_{\POPT_{t+1}}\|(\PCE_{t+1})^{-1}\|\|\PCE_{t+1}\| \leq \kappa_{\POPT_{t+1}}(\|(\POPT_{t+1})^{-1}\| +
\frac{\kappa_{\POPT_{t+1}}}{\|\POPT_{t+1}\|^2 -
                        \kappa_{\POPT_{t+1}}})(\|\POPT_{t+1}\| +
                        \frac{1}{\|\POPT_{t+1}\|}) \\
  &= \kappa_{\POPT_{t+1}}^2 +
    \frac{\kappa_{\POPT_{t+1}}^2}{\|\POPT_{t+1}\|^2} + \frac{\kappa_{\POPT_{t+1}}^2}{\|\POPT_{t+1}\|^2 -
                        \kappa_{\POPT_{t+1}}}(\|\POPT_{t+1}\| +
                        \frac{1}{\|\POPT_{t+1}\|})
\end{align*}
Denoting $c_{\POPT_{t+1}}$ as the right hand side expression in the
above inequality we get the desired result.The example that
realizes the upper bound is given in Appendix~\ref{sec:scalar-example-that}.

\end{proof}

\section{Note on Assumption~\ref{assumption:psd}}
\label{sec:assumpt-refass}

Consider the cost-to-go matrix $\PILC_t$ given by
\begin{align*}
  \PILC_{t} &= Q + \Ahat_t^{T}\PILC_{{t+1}}A_t - \Ahat_t^{T}\PILC_{t+1}B_t(R + \Bhat_t^{T}\PILC_{{t+1}}B_t)^{-1}\Bhat_t^{T}\PILC_{t+1}A_t \\
  &= Q + \Ahat_t^{T}\PILC_{t+1}(I + B_tR^{-1}\Bhat_t^{T}\PILC_{t+1})^{-1}A_t
\end{align*}
and the cost-to-go from any state $x$ is given by
\begin{align*}
  V_t(x) = x^T\PILC_tx
\end{align*}
Since this is a quadratic, for it to be convex (and thus, have a
minima) we require the leading 
coefficient to be positive semi-definite. In other words, $\PILC_t$
should have eigenvalues with non-negative real parts. Assuming
$\PILC_{t+1}$ to be positive semi-definite, and observing the fact
that $Q$ is a positive semi-definite matrix, we require that
$B_tR^{-1}\Bhat_t^T$ to have eigenvalues with non-negative real parts
for $\PILC_t$ to be positive semi-definite. Note that this is
trivially satisfied for \MM{} as the leading coefficient there
contains a similar term $B_tR^{-1}B_t^T$ which is positive
semi-definite.

Intuitively, if $B_tR^{-1}\Bhat_t^T$ does not have eigenvalues with
non-negative real parts, then the resulting quadratic cost-to-go
function need not be convex, and \ILC{} will not converge.

\section{Iterative Learning Control Results}
\label{sec:iter-learn-contr}

Our first lemma derives a similar result as
Lemma~\ref{lemma:ce-quadratic} but for the iterative learning control
setting,
\begin{lemma}
  \label{lemma:ilc-quadratic}
  Given functions $f_{1}(x, u)$ and $f_{2}(x, u)$ such that
  $\nabla_{u} f_{1}(x, u) = (B_{1}^{T}P_{1}B_{1} + R)u + B_{1}^{T}P_{1}A_{1}x$
  and
  $\nabla_{u} f_{2}(x, u) = (B_{2}^{T}P_{2}B_{1} + R)u + B_{2}^{T}P_{2}A_{1}x$
  where $R, P_{1}, P_{2}$ are positive-definite matrices. Let $K_{1}$ and
  $K_{2}$ be unique matrices such that $\nabla_{u} f_{1}(x, K_{1}x) = 0$ and
  $\nabla_{u} f_{2}(x, K_{2}x) = 0$ for any
  vector $x$. Also,
  denote $\Gamma = 1 + \max\{||A_1||, ||B_1||, ||P_1||,
  ||K_1||\}$. Suppose there exists $\epsA, \epsB, \epsP > 0$ (and $<\Gamma$) such that
  $||A_1 - A_2|| \leq \epsA$, and $||B_1 - B_2|| \leq
  \epsB$, and $||P_1 - P_2|| \leq \epsP$. Then we have,
  \begin{equation}
    \label{eq:81}
    \|K_{1} - K_{2}\| \leq \frac{2\Gamma^3(\epsB + 2\epsP)}{\ubar{\sigma}(R)}
  \end{equation}
\end{lemma}
\begin{proof}
  Let us bound the difference $\|\nabla_{u} f_{1}(x,u) - \nabla_{u} f_{2}(x, u)\|$
  by bounding each term separately. First consider the term
  \begin{align*}
    \|B_{1}^{T}P_{1}B_{1} - B_{2}^{T}P_{2}B_{1}\| &= \|(B_{1} - B_{2})^{T}P_{1}B_{1} + B_{2}^{T}(P_{1} - P_{2})B_{1}\| \\
                                                  &\leq \Gamma^{2}\epsB + \Gamma(\Gamma + \epsB)\epsP \\
    &\leq \Gamma^2(\epsB + 2\epsP)
  \end{align*}
  where we used the fact that $\|B_{2}\| \leq \Gamma + \epsB$.
  We can similarly bound the term
  \begin{align*}
    \|B_{1}^{T}P_{1}A_{1} - B_{2}^{T}P_{2}A_{1}\| \leq \Gamma^{2}(\epsB + 2\epsP)
  \end{align*}

  Thus, we have for any vector $x$ such that $\|x\| \leq 1$
  \begin{align*}
    \|\nabla_{u} f_{1}(x,u) - \nabla_{u} f_{2}(x, u)\| \leq \Gamma^{2}(\epsB + 2\epsP)(\|u\| + 1)
  \end{align*}
  Substituting $u=u_{1}$ we get
  \begin{align*}
    \|\nabla_{u} f_{2}(x, u_{1})\| \leq \Gamma^{2}(\epsB + 2\epsP)(\|u_{1}\| + 1)
  \end{align*}

  We can bound $\|u_{1}\| \leq \|K_{1}\|\|x\| \leq \|K_{1}\| \leq \Gamma$. Then
  from Lemma~\ref{lemma:1} we have,
  \begin{align*}
    \|u_{1} - u_{2}\| &\leq \frac{\Gamma^{2}(\epsB + 2\epsP)(\Gamma+1)}{\ubar{\sigma}(R)} \\
    \|K_{1} - K_{2}\| &\leq \frac{2\Gamma^{3}(\epsB + 2\epsP)}{\ubar{\sigma}(R)}
  \end{align*}
\end{proof}

Now we will prove Lemma~\ref{lemma:ilc},
\ilcLemma*
\begin{proof}
  Use Assumption~\ref{assumption:singularvalue} and
  Lemma~\ref{lemma:ilc-quadratic} for $t = 0, \cdots, H-1$ with 
  $\epsP = \fILC_{t+1}(\epsA, \epsB)$ and choosing $\epsilon_t =
  \max\{\epsA, \epsB, \fILC_{t+1}(\epsA, \epsB)\}$.
\end{proof}

Our final task is to prove Theorem~\ref{theorem:ilc},
\ilcTheorem*
\begin{proof}
  We know $\PILC_{t}$ satisfies,
\begin{align*}
  \PILC_{t} &= Q + \Ahat_t^{T}\PILC_{{t+1}}A_t - \Ahat_t^{T}\PILC_{t+1}B_t(R + \Bhat_t^{T}\PILC_{{t+1}}B_t)^{-1}\Bhat_t^{T}\PILC_{t+1}A_t \\
  &= Q + \Ahat_t^{T}\PILC_{t+1}(I + B_tR^{-1}\Bhat_t^{T}\PILC_{t+1})^{-1}A_t
\end{align*}
where we used the matrix inversion lemma.

Consider the difference,
\begin{align*}
  \POPT_{t} - \PILC_{t} &= A_t^{T}\POPT_{t+1}(I + B_tR^{-1}B_t^{T}\POPT_{t+1})^{-1}A_t - \Ahat_t^{T}\PILC_{t+1}(I + B_tR^{-1}\Bhat_t^{T}\PILC_{t+1})^{-1}A_t \\
                      &= A^{T}\POPT_{t+1}(I + B_tR^{-1}B_t^{T}\POPT_{t+1})^{-1}A_t - \Ahat_t^{T}\POPT_{t+1}(I + B_tR^{-1}\Bhat_t^{T}\POPT_{t+1})^{-1}A_t \\
  &+ \Ahat_t^{T}\left(\POPT_{t+1}(I + B_tR^{-1}\Bhat_t^{T}\POPT_{t+1})^{-1} - \PILC_{t+1}(I + B_tR^{-1}\Bhat_t^{T}\PILC_{t+1})^{-1}\right)A_t
\end{align*}

Here again we can use Lemma~\ref{lemma:mania-appendix-original} with
$S = B_tR^{-1}B_t^{T}$, $\Shat = B_tR^{-1}\Bhat_t^{T}$, $Q = \POPT_{t+1}$ and observing that
$||\Shat - S|| \leq ||B_t||||R^{-1}||\epsB$ to get
\begin{align*}
  ||\PILC_{t} - \POPT_{t}|| &\leq ||A_t||^{2}||\POPT_{t+1}||^{2}||B_t||||R^{-1}||\epsB + ||A_t||||\POPT_{t+1}||\epsA \\
  &+ ||A_t||||\Ahat_t||||\POPT_{t+1}(I + B_tR^{-1}\Bhat_t^{T}\POPT_{t+1})^{-1} - \PILC_{t+1}(I + B_tR^{-1}\Bhat_t^{T}\PILC_{t+1})^{-1}||
\end{align*}

Here again we use Lemma~\ref{lemma:stackexchange} to bound the second expression
giving us
\begin{align*}
  ||\PILC_{t} - \POPT_{t}|| &\leq ||A_t||^{2}||\POPT_{t+1}||^{2}||B_t||||R^{-1}||\epsB + ||A_t||||\POPT_{t+1}||\epsA \\
  &+ ||A_t||||\Ahat_t|| ||(I + B_tR^{-1}\Bhat_t^{T}\POPT_{t+1})^{-1}|| ||\PILC_{t+1} - \POPT_{t+1}|| ||(I + B_tR^{-1}\Bhat_t^{T}\PILC_{t+1})^{-1}||
\end{align*}

This can be rewritten as the final bound,
\begin{align}
  \label{eq:91}
  ||\POPT_{t} - \PILC_{t}|| &\leq ||A_t||^{2}||\POPT_{t+1}||^{2}||B_t||||R^{-1}||\epsB + ||A_t||||\POPT_{t+1}||\epsA \nonumber\\
  &+ (||A_t||^{2} + \epsA||A_t||) \kappa_{\POPT_{t+1}}\kappa_{\PILC_{t+1}} ||\POPT_{t+1} - \PILC_{t+1}||
\end{align}

The constant $c_{\POPT_{t+1}}$ can be derived very similarly as we
have done in the proof of Theorem~\ref{theorem:ce}. The example that
realizes the upper bound is given in Appendix~\ref{sec:scalar-example-that}.
\end{proof}

\section{Scalar Example that Realizes Upper Bounds}
\label{sec:scalar-example-that}

\subsection{General Formulation}
\label{sec:general-formulation}
Consider a $1$D linear dynamical system given by,
\begin{equation}
  \label{eq:10}
  x_t = ax_t + bu_t
\end{equation}
where $x_t, u_t, a, b \in \reals$. The cost function is given by,
\begin{equation}
  \label{eq:11}
  V_0(x_0) = \sum_{t=0}^{H-1} qx_t^2 + ru_t^2 + qx_H^2
\end{equation}
We are given access to an approximate model specified using $\ahat,
\bhat \in \reals$.

The optimal cost-to-go is specified using
\begin{align}
  \label{eq:12}
  &\popt_H = q \\
  &\popt_t = q + \frac{a^2\popt_{t+1}}{1 + b^2r^{-1}\popt_{t+1}} = q + \frac{a^2r\popt_{t+1}}{r + b^2\popt_{t+1}}
\end{align}

For \MM{}, the cost-to-go is specified using
\begin{align}
  \label{eq:13}
  &\pce_H = q \\
  &\pce_t = q + \frac{\ahat^2r\pce_{t+1}}{r + \bhat^2\pce_{t+1}}
\end{align}

For ILC, the cost-to-go is specified using
\begin{align}
  \label{eq:8}
  &\pilc_h = q \\
  &\pilc_t = q + \frac{a\ahat r\pilc_{t+1}}{r + b\bhat\pilc_{t+1}}
\end{align}

In the next two subsections, we will show that an example dynamical
system where $\bhat = 0$, i.e. the approximate model thinks that the
system is not controllable will realize the worst case upper bounds
for both \MM{} and ILC as presented in Theorems~\ref{theorem:ce}
and~\ref{theorem:ilc} respectively.

\subsection{Optimal Control with Misspecified Model}
\label{sec:cert-equiv-contr-1}

Consider the difference
\begin{align*}
  \popt_t - \pce_t &= \frac{a^2r\popt_{t+1}}{r + b^2\popt_{t+1}} -
                       \frac{\ahat^2r\pce_{t+1}}{r +
                       \bhat^2\pce_{t+1}} \\
  &= \left( \frac{a^2r\popt_{t+1}}{r + b^2\popt_{t+1}} -
    \frac{\ahat^2r\popt_{t+1}}{r + \bhat^2\popt_{t+1}} \right) +
    \left( \frac{\ahat^2r\popt_{t+1}}{r + \bhat^2\popt_{t+1}} - \frac{\ahat^2r\pce_{t+1}}{r +
                       \bhat^2\pce_{t+1}} \right)
\end{align*}
Let us look at each term separately. The first term can be simplified
as
\begin{equation}
  \label{eq:17}
  \left( \frac{a^2r\popt_{t+1}}{r + b^2\popt_{t+1}} -
    \frac{\ahat^2r\popt_{t+1}}{r + \bhat^2\popt_{t+1}} \right) =
                                                                 \frac{\popt_{t+1}(a^2-
                                                                 \ahat^2)}{1
                                                                 +
                                                                 \bhat^2r^{-1}\popt_{t+1}}
                                                                 +
                                                                 \frac{a^2r^{-1}(\bhat^2
                                                                 -
                                                                 b^2)(\popt_{t+1})^2}{(1
                                                                 +
                                                                 b^2r^{-1}\popt_{t+1})(1
                                                                 +
                                                                 \bhat^2r^{-1}\popt_{t+1})}                                                     
\end{equation}
Similarly, the second term can be simplified as
\begin{equation}
  \label{eq:9}
  \left( \frac{\ahat^2r\popt_{t+1}}{r + \bhat^2\popt_{t+1}} - \frac{\ahat^2r\pce_{t+1}}{r +
                       \bhat^2\pce_{t+1}} \right) =
                   \frac{\ahat^2(\popt_{t+1} - \pce_{t+1})}{(1 +
                     \bhat^2r^{-1}\popt_{t+1})(1 + \bhat^2r^{-1}\pce_{t+1})}
\end{equation}

Now, consider the example dynamical system where $a - \ahat = \epsa$,
$b - \bhat = \epsb$, and $\bhat = 0$. Our upper bound in
Theorem~\ref{theorem:ce} states that,
\begin{equation}
  \label{eq:16}
  |\popt_{t} - \pce_t| \leq a^2r^{-1}(\popt_{t+1})^2(2b\epsb +
  \epsb^2) + \popt_{t+1}(2a\epsa + \epsa^2) + (a +
  \epsa)^2|\popt_{t+1} - \pce_{t+1}|
\end{equation}
For the example system equation~\eqref{eq:17} simplifies to,
\begin{align*}
  \frac{\popt_{t+1}(a^2-
                                                                 \ahat^2)}{1
                                                                 +
                                                                 \bhat^2r^{-1}\popt_{t+1}}
                                                                 +
                                                                 \frac{a^2r^{-1}(\bhat^2
                                                                 -
                                                                 b^2)(\popt_{t+1})^2}{(1
                                                                 +
                                                                 b^2r^{-1}\popt_{t+1})(1
                                                                 +
                                                                 \bhat^2r^{-1}\popt_{t+1})}
  = \popt_{t+1}(2a\epsa + \epsa^2)
                                                                 +
                                                                 \frac{a^2r^{-1}(\popt_{t+1})^2(2b\epsb
  + \epsb^2)}{(1
                                                                 +
                                                                 b^2r^{-1}\popt_{t+1})}
\end{align*}
which matches the first two terms in the upper bound
(equation~\eqref{eq:16}) upto a constant. Now, let's look at how
equation~\eqref{eq:9} simplifies
\begin{align*}
  \frac{\ahat^2(\popt_{t+1} - \pce_{t+1})}{(1 +
                     \bhat^2r^{-1}\popt_{t+1})(1 +
  \bhat^2r^{-1}\pce_{t+1})} = (a + \epsa)^2(\popt_{t+1} - \pce_{t+1})
\end{align*}
which matches the last term in the upper bound
(equation~\eqref{eq:16}) exactly. Thus, we found an example where
$|\popt_t - \pce_t|$ matches the upper bound specified in
Theorem~\ref{theorem:ce} upto a constant.

\subsection{Iterative Learning Control}
\label{sec:iter-learn-contr-2}

Consider the difference
\begin{align*}
  \popt_t - \pilc_t &= \frac{a^2r\popt_{t+1}}{r + b^2\popt_{t+1}} -
                       \frac{a\ahat r\pilc_{t+1}}{r +
                      b\bhat\pilc_{t+1}} \\
  &= \left( \frac{a^2r\popt_{t+1}}{r + b^2\popt_{t+1}} -
    \frac{a\ahat r\popt_{t+1}}{r + b\bhat\popt_{t+1}} \right) + \left(
    \frac{a\ahat r\popt_{t+1}}{r + b\bhat\popt_{t+1}} - \frac{a\ahat r\pilc_{t+1}}{r +
                      b\bhat\pilc_{t+1}}\right)
\end{align*}
Once again let us look at each term separately. The first term can be
simplified as
\begin{equation}
  \label{eq:20}
  \left( \frac{a^2r\popt_{t+1}}{r + b^2\popt_{t+1}} -
    \frac{a\ahat r\popt_{t+1}}{r + b\bhat\popt_{t+1}} \right) =
  \frac{a\popt_{t+1}(a - \ahat)}{(1 + b\bhat r^{-1}\popt_{t+1})} +
  \frac{a^2br^{-1}(\popt_{t+1})^2(\bhat - b)}{(1 + b^2r^{-1}\popt_{t+1})(1
    + b\bhat r^{-1}\popt_{t+1})}
\end{equation}
Similarly, the second term can be simplified as
\begin{equation}
  \label{eq:18}
  \left(
    \frac{a\ahat r\popt_{t+1}}{r + b\bhat\popt_{t+1}} - \frac{a\ahat r\pilc_{t+1}}{r +
                      b\bhat\pilc_{t+1}}\right) =
                  \frac{a\ahat(\popt_{t+1} - \pilc_{t+1})}{(1 + b\bhat
                    r^{-1}\popt_{t+1})(1 + b\bhat r^{-1}\pilc_{t+1})}
\end{equation}
Similar to \MM{} in the previous section, consider the example dynamical
system where $a - \ahat = \epsa$, $b - \bhat = \epsb$ and $\bhat =
0$. Our upper bound in Theorem~\ref{theorem:ilc} states that
\begin{equation}
  \label{eq:19}
  |\popt_t - \pilc_t| \leq a^2(\popt_{t+1})^2br^{-1}\epsb +
  a\popt_{t+1}\epsa + a(a + \epsa)|\popt_{t+1} - \pilc_{t+1}|
\end{equation}
For the example dynamical system, equation~\eqref{eq:20} simplifies to
\begin{align*}
  \frac{a\popt_{t+1}(a - \ahat)}{(1 + b\bhat r^{-1}\popt_{t+1})} +
  \frac{a^2br^{-1}(\popt_{t+1})^2(\bhat - b)}{(1 + b^2r^{-1}\popt_{t+1})(1
    + b\bhat r^{-1}\popt_{t+1})} = a\popt_{t+1}\epsa +
  \frac{a^2(\popt_{t+1})^2br^{-1}\epsb}{(1 + b^2r^{-1}\popt_{t+1})}
\end{align*}
which matches the first two terms in the upper bound
(equation~\eqref{eq:19}) upto a constant. Now, let's look at how
equation~\eqref{eq:18} simplifies
\begin{align*}
  \frac{a\ahat(\popt_{t+1} - \pilc_{t+1})}{(1 + b\bhat
                    r^{-1}\popt_{t+1})(1 + b\bhat r^{-1}\pilc_{t+1})}
  = a(a + \epsa)(\popt_{t+1} - \pilc_{t+1})
\end{align*}
which matches the last term in the upper bound
(equation~\eqref{eq:19}) exactly. Thus, we found that the same example
also matches the upper bound specified in Theorem~\ref{theorem:ilc}
upto a constant.

\section{Experiment Details}
\label{sec:experiment-details}

\subsection{Linear Dynamical System with Approximate Model}
\label{sec:line-dynam-syst-1}

We use a horizon $H = 10$ and initial state $x_0 = \begin{bmatrix}0.1
  \\ 0.1 \end{bmatrix}$.

\subsection{Nonlinear Inverted Pendulum with Misspecified Mass}
\label{sec:nonl-invert-pend}

For the second experiment, we use the nonlinear dynamical system of an inverted
pendulum. The state space is specified by $x =
\begin{bmatrix}
  \theta \\
  \dot{\theta}
\end{bmatrix} \in \reals^{2}
$ where $\theta$ is the angle between the pendulum and the vertical axis. The
control input is $u = \tau \in \reals$ specifying the torque $\tau$ to be
applied at the base of the pendulum. The dynamics of the system are
given by the ODE,
$\ddot{\theta} = \frac{\bar{\tau}}{m\ell^{2}} - \frac{g\sin(\theta)}{\ell}$
where $m$ is the mass of the pendulum, $\ell$ is the length of the pendulum, $g$
is the acceleration due to gravity, and
$\bar{\tau} = \max(\tau_{\min}, \min(\tau_{\max}, \tau))$ is the clipped torque
based on torque limits. We use $\ell = 1$m, $\tau_{\max} = 8$Nm,
$\tau_{\min} = -8$Nm, and $m = 1$kg.

We use a per time step cost function defined as $c(\theta, \tau) =
0.1\tau^{2} + \theta^{2}$ 
where $\theta \in [-\pi, \pi]$, an initial state $x_{0} =
\begin{bmatrix}
  \frac{\pi}{2} \\ 0.5
\end{bmatrix}
$, and a horizon $H = 20$. For all algorithms, we start with an
initial control sequence consisting of zero torques for the entire horizon.

\subsection{Nonlinear Planar Quadrotor Control in Wind}
\label{sec:nonl-plan-quadr-1}

In our final experiment, we compare \MM{} and \ILC{} on a planar quadrotor control
task in the presence of wind. The quadrotor is controlled using two propellers
that provide upward
thrusts $(u_{1}, u_{2})$ and allows movement in the $3$D planar space
described as
$(p_{x}, p_{y}, \theta)$ where $p_{x}, p_{y}$ are X, Y positions, and
$\theta$ is the yaw of the quadrotor. The dynamics of the planar quadrotor
is specified
using a state vector $x \in \reals^{6}$, control input $u \in \reals^{2}$ as
\begin{align*}
  x =
  \begin{bmatrix}
    p_{x} \\ p_{y} \\ \theta \\ \dot{p}_{x} \\ \dot{p}_{y} \\ \dot{\theta}
  \end{bmatrix}, u =
  \begin{bmatrix}
    u_{1} \\ u_{2}
  \end{bmatrix},
  \dot{x} =
  \begin{bmatrix}
    \dot{p}_{x} \\ \dot{p}_{y} \\ \dot{\theta} \\ \frac{1}{m}(u_{1} + u_{2})\sin(\theta) \\ \frac{1}{m}(u_{1} + u_{2})\cos(\theta) - g \\ \frac{\ell}{2J}(u_{2} - u_{1})
  \end{bmatrix}
\end{align*}
where $m$ is the mass of the quadrotor, $\ell$ is the distance between the
propellers, $g$ is acceleration due to gravity, and $J$ is the moment of
inertia of the quadrotor. We use $m = 1$kg, $\ell = 0.3$m, and
$J = 0.2m\ell^{2}$. The objective of the task is to move the quadrotor from an initial
state $x_0$
at $(-3, 1)$ with zero velocity to a final state $x_f$ at $(3, 1)$ with zero
velocity. This is achieved using the per time-step cost function
$c(x, u) = (x - x_{f})^{T}Q(x - x_{f}) + (u - u_{h})^{T}R(u - u_{h})$
where $u_{h} = [\frac{1}{2}mg, \frac{1}{2}mg]$ are the hover controls. We use a
horizon of $H = 60$ with a step size of $0.025$ for RK4 integration.

\end{document}